\documentclass[journal]{IEEEtran}
\usepackage{amsmath,amsfonts}
\usepackage{algorithmic}
\usepackage{algorithm}
\usepackage{array}
\usepackage[caption=false,font=normalsize,labelfont=rm,textfont=rm]{subfig}
\usepackage{textcomp}
\usepackage{stfloats}
\usepackage{url}
\usepackage{verbatim}
\usepackage{graphicx}
\usepackage{cite}
\usepackage{balance}
\usepackage{bm}
\usepackage{makecell}
\usepackage{xr-hyper}
\usepackage{hyperref}
\usepackage{utfsym}
\IfFileExists{orcidlink.sty}{%
    \usepackage{orcidlink}
}{%
    \newcommand{\orcidlink}[1]{}
}
\usepackage{multirow}
\usepackage{booktabs}
\newtheorem{proposition}{Proposition}

\begin{document}

\title{Zero-Shot Scalable Resilience in UAV Swarms: \\A Decentralized Imitation Learning Framework with Physics-Informed Graph Interactions}


\author{Huan Lin~\orcidlink{0009-0002-7217-6139}, Chun Yang, Lianghui Ding~\orcidlink{0000-0002-3231-3613 },~\IEEEmembership{Member,~IEEE}, and Yulin Qiu
\thanks{This paper was supported in part by Shanghai Key Laboratory Funding under Grant STCSM15DZ2270400, and in part by the Program for Professor of Special Appointment (Eastern Scholar) at Shanghai Institutions of Higher Learning. (\textit{Corresponding author: Lianghui Ding.})}
\thanks{Huan Lin, Chun Yang, Lianghui Ding, and Yulin Qiu are with the Institute of Image Communication and Network Engineering, School of Integrated Circuits, Shanghai Jiao Tong University, Shanghai 200240, China. Email: lhzt715@sjtu.edu.cn; yangchun123@sjtu.edu.cn; lhding@sjtu.edu.cn; qyl7@sjtu.edu.cn.}
}



\maketitle

\begin{abstract}
Large-scale UAV failures can fragment a swarm network into disconnected sub-networks, making connectivity recovery both urgent and difficult. Centralized recovery methods rely on global topology information and become difficult to deploy after severe fragmentation, while decentralized heuristics and existing learning-based methods still degrade under wide variations in swarm scale and damage severity. We present Physics-informed Graph Adversarial Imitation Learning (PhyGAIL), a recovery framework trained under centralized training and decentralized execution that adapts graph-based MARL at three coupled stages. First, bounded local interaction graphs provide scale-stable policy inputs from heterogeneous local observations. Second, a physics-gated graph neural network explicitly models attraction and repulsion interactions for reconnection and collision avoidance. Third, scenario-adaptive imitation learning combines an offline composite expert database, dense GAIL rewards, and expert-normalized temporal rewards to stabilize training under sparse terminal feedback and variable recovery horizons. Structural analysis further explains why PhyGAIL supports scale-free adaptation and stable recovery learning in large damaged swarms. In large-scale simulations, a policy trained on 20-UAV swarms transfers directly to swarms of up to 500 UAVs without fine-tuning and achieves superior reconnection reliability, recovery efficiency, motion safety, and runtime efficiency over representative baselines.
\end{abstract}


\section{Introduction}
\IEEEPARstart{R}{ecent} advances in sensing, communication, and autonomous control have made unmanned aerial vehicle (UAV) swarms increasingly practical for surveillance, search and rescue, and emergency communication \cite{javed2024state, horyna2023decentralized}. In these applications, multiple UAVs cooperate through a UAV swarm network (USNET). The network, however, is fragile. Environmental hazards, coordinated failures, or hostile interference may disable many UAVs in a single event, splitting the swarm into disconnected sub-networks and triggering the communication network split (CNS) problem \cite{phadke2022towards}.

The CNS problem has long been studied through mobility-assisted connectivity restoration in wireless sensor, mobile ad hoc, and UAV networks \cite{lee2010recovery, lee2015connectivity, shankar2008hybrid, zhang2018autonomous, akkaya2013handling, chouikhi2017centralized}. For large-scale UAV swarms, existing solutions are broadly either centralized or decentralized, with most relying on hand-crafted recovery rules. Centralized recovery plans motion from global topology information, which becomes communication-intensive and increasingly impractical when severe fragmentation prevents efficient global state collection. Decentralized heuristic methods act on local observations and are easier to deploy after a split, but they typically follow fixed recovery rules that adapt poorly to large variations in topology and damage severity.

Multi-agent reinforcement learning (MARL) offers a more adaptive alternative \cite{li2025fast, wang2025topology} because CNS recovery is a closed-loop continuous-control problem rather than a one-shot path planning task. Each UAV acts from partial local observations, and its motion changes inter-UAV distances, reshapes the communication graph, and affects the next-step observations. A learned decentralized policy is therefore attractive because it can adapt online as fragmentation evolves, instead of relying on fixed motion rules or repeatedly solving a centralized recovery plan after every topology change.

This perspective reveals three coupled bottlenecks in adapting graph MARL to damaged USNETs. First, if the graph receptive field grows with swarm size, the number of neighbors, the message range, and the aggregated feature distribution all change, so a policy trained on small swarms can face an out-of-distribution local decision context in larger swarms \cite{yehudai2021local}. Second, local neighbors can imply opposite interaction semantics, pulling disconnected subnets together for reconnection while pushing nearby UAVs apart for collision avoidance. Existing graph aggregation, however, mainly learns which neighbors matter, rather than whether they should induce attraction or repulsion. Third, recovery learning becomes increasingly difficult because success feedback is sparse and heavier damage creates much longer, harder reconnection episodes. Successful reconnection is observed only after separated components merge, so before success the main feedback is often the step penalty, which can dominate long and difficult episodes. Imitation learning \cite{ho2016generative} can densify guidance, but it must still account for scenario difficulty.

We address these issues with a Physics-informed Graph Adversarial Imitation Learning algorithm (PhyGAIL), a decentralized recovery framework trained under centralized training and decentralized execution (CTDE). Rather than stacking CTDE, graph neural encoding, and imitation learning as separate components, PhyGAIL adapts graph MARL at three coupled stages: bounded graph construction, physics-gated interaction modeling, and expert-guided imitation training. A bounded local interaction graph keeps the policy input scale-stable and limits its dependence on global swarm size. Physics-informed graph neural network (PhyGNN) uses attraction and repulsion gates to encode reconnection and safety cues directly in message passing, so that the policy can distinguish these opposite interaction semantics directly through the architecture rather than from sparse environmental feedback alone. An expert-guided imitation strategy combines an offline composite expert database, dense GAIL rewards, and expert-normalized temporal rewards to stabilize training across variable damage levels without giving the decentralized test-time actor extra global information. Structural analysis of the proposed architecture provides theoretical support for PhyGAIL's scale-free adaptation and stable recovery learning. In large-scale simulations with 50 disconnected post-damage cases per setting, PhyGAIL reduces the average recovery time by 24.1\% relative to the best non-PhyGAIL baseline, while maintaining strong reconnection reliability, motion safety, and runtime efficiency.

The main contributions of this paper are as follows:
\begin{enumerate}
\item \textbf{A scale-invariant decentralized recovery formulation for zero-shot transfer.}
We enable zero-shot transfer, i.e., direct deployment without retraining or fine-tuning, from 20-UAV training swarms to networks of up to 500 UAVs. This is achieved by formulating CNS recovery as a decentralized continuous-control problem under the CTDE paradigm and constructing bounded local interaction graphs for policy execution.

\item \textbf{A physics-gated graph neural network for resolving neighbor interaction ambiguity.}
We propose PhyGNN to separate attraction from repulsion through dual-gated message passing on local interaction edges. This design provides a physically meaningful coordination prior for reconnection and collision avoidance, reducing the burden of learning opposite motion semantics from sparse delayed rewards.

\item \textbf{An expert-guided imitation learning strategy for stable variable-damage training.}
We construct an offline expert database by running multiple baseline methods across diverse damage severities and selecting the fastest successful trajectory for each scenario. Based on this database, PhyGAIL combines dense GAIL rewards with expert-normalized temporal rewards, improving credit assignment and reward balance across variable-length recovery tasks.
\end{enumerate}

The remainder of this paper is organized as follows. Section \uppercase\expandafter{\romannumeral2} reviews the related work. Section \uppercase\expandafter{\romannumeral3} formulates the system model, conceptualizing the USNET graph and the CNS problem. Section \uppercase\expandafter{\romannumeral4} details the architectural design and learning paradigm of the proposed algorithm. Extensive simulation results and performance evaluations are discussed in Sections \uppercase\expandafter{\romannumeral5}, followed by the conclusions in Section \uppercase\expandafter{\romannumeral6}.

\section{Related Work}

Existing mobility-based solutions to the Communication Network Split (CNS) problem can be grouped into three categories: centralized recovery algorithms, decentralized heuristic methods, and learning-based approaches. Table~\ref{tab:related_work_comparison} summarizes the main CNS recovery families using labels that are detailed subsequently in this section.

\subsection{Centralized Optimization Algorithms}
Early studies treated connectivity restoration as a centralized optimization problem and assumed whole-swarm topology information at decision time. Senel and Younis \cite{senel2011optimized}, Chouikhi \textit{et al.} \cite{chouikhi2017centralized}, and Zhang \textit{et al.} \cite{zhang2019hybrid} respectively studied relay placement, critical-node replacement, and hybrid optimization in partitioned networks. More recently, graph-learning methods such as CR-MGC \cite{mou2022resilient} and DEMD \cite{lin2024multihop} further improved centralized recovery. Mou \textit{et al.} used graph convolution in CR-MGC to iteratively generate recovery motions from the post-damage global topology. Lin \textit{et al.} further proposed DEMD to inject explicit damage cues into global graph reasoning for multihop reconnection. Their common limitation is that they still rely on post-damage whole-swarm information and centralized inference over the full graph. Such dependence is effective when global state can be collected reliably, but becomes harder to sustain after severe fragmentation, precisely where decentralized recovery is most needed.

\subsection{Decentralized Heuristic Methods}
After a communication network split, complete global topology is often unavailable or unreliable. Decentralized heuristic methods therefore rely on local or intra-subnet information. One line uses critical-node replacement, such as volunteer-instigated relocation \cite{imran2010volunteer}, LeDir \cite{abbasi2012recovering}, and DCRMF \cite{zhang2018autonomous}. These methods are simple and effective in single-node or mildly damaged settings, but repeated relocation can trigger cascaded motion and coverage loss under more severe failures. Another line uses reference-direction or potential-field coordination, such as HERO \cite{mi2011hero} and SIDR \cite{chen2020sidr}. These methods are lightweight and deployment-friendly, but their recovery logic remains rule-based and becomes less adaptive under large topology and damage variations.

\subsection{Reinforcement Learning-Based Approaches}
Reinforcement learning (RL) is attractive because CNS recovery is inherently a sequential continuous-control problem under time-varying topology. Each motion decision changes inter-UAV geometry and future connectivity, so the policy must respond to an evolving recovery state rather than execute a fixed relocation rule. Wang \textit{et al.} \cite{wang2025topology} proposed GDR-TS, which learns topology-sensitive recovery policies from global swarm information through centralized graph construction. Li \textit{et al.} \cite{li2025fast} proposed CR-MA, which follows a CTDE-style MARL design and combines artificial potential fields with learned control to support more decentralized execution and explicit collision avoidance.

The baselines mentioned above suggest that graph-based RL, CTDE, and collision-aware priors are useful for CNS recovery. However, they do not fully resolve large-scale adaptation. GDR-TS remains tied to global graph construction, whereas CR-MA supports more decentralized execution but still inherits scale-dependent local graph inputs and relies heavily on hand-crafted interaction bias. More broadly, sparse terminal rewards and variable recovery horizons motivate stronger imitation-learning guidance \cite{ho2016generative}. These gaps motivate PhyGAIL, whose design targets graph construction, interaction modeling, and training stability jointly.

\begin{table*}[!t]
\centering
\small
\caption{Representative CNS recovery methods.}
\label{tab:related_work_comparison}
\renewcommand{\arraystretch}{1.1}
\setlength{\tabcolsep}{4pt}
\resizebox{0.98\textwidth}{!}{
\begin{tabular}{lcccccc}
\toprule
Method & Year & Observation Scope$^{\dagger}$ & Swarm Scale & Damage Scale$^{\ddagger}$ & Collision Avoidance & Recovery Paradigm \\
\midrule
\multicolumn{7}{l}{\textit{Centralized optimization}} \\
CIST \cite{senel2011optimized} & 2011 & Global & 15 & Single-node & \usym{2613} & Relay node placement \\
RNFR \cite{chouikhi2017centralized} & 2017 & Global & 250 & Single-node & \usym{2613} & Critical-node replacement \\
HPC-CR \cite{zhang2019hybrid} & 2019 & Global & 140 & Single-node & \usym{2613} & Critical-node replacement \\
CR-MGC \cite{mou2022resilient} & 2022 & Global & 200 & Large-scale & \usym{2613} & GNN-based online iteration \\
DEMD \cite{lin2024multihop} & 2024 & Global & 200 & Large-scale & \usym{2613} & GNN-based online iteration \\
\midrule
\multicolumn{7}{l}{\textit{Decentralized heuristic}} \\
HERO \cite{mi2011hero} & 2011 & Intra-subnet & 32 & Small-scale & $\checkmark$ & Potential field coordination \\
LeDir \cite{abbasi2012recovering} & 2013 & Local & 100 & Single-node & \usym{2613} & Critical-node replacement \\
DCRMF \cite{zhang2018autonomous} & 2018 & Local & 80 & Single-node & \usym{2613} & Critical-node replacement \\
SIDR \cite{chen2020sidr} & 2020 & Intra-subnet & 100 & Large-scale & $\checkmark$ & Potential field coordination \\
\midrule
\multicolumn{7}{l}{\textit{Reinforcement learning-based}} \\
GDR-TS \cite{wang2025topology} & 2025 & Global & 100 & Large-scale & \usym{2613} & GNN-based online CTCE RL \\
CR-MA \cite{li2025fast} & 2025 & Intra-subnet & 20 & Large-scale & $\checkmark$ & CTDE RL + Potential field \\
PhyGAIL (ours) & - & Intra-subnet & 500 & Large-scale & $\checkmark$ & GNN-based zero-shot CTDE RL\\
\bottomrule
\end{tabular}}
\vspace{0.4em}
\parbox{0.98\textwidth}{\footnotesize $^{\dagger}$ Observation scope is grouped as global (whole-swarm topology access), intra-subnet (information shared within the surviving connected component), and local (purely node-neighbor observations).\\$^{\ddagger}$ Damage scale is grouped as single-node, small-scale ($\leq 20\%$ failed nodes), and large-scale ($\geq 90\%$ failed nodes).}
\end{table*}


\section{System model}
This section defines the mathematical models for UAV swarm connectivity restoration. We first introduce the communication and network models, followed by the mobility model. Subsequently, the damage model is described, and the decentralized recovery process is formulated as a constrained problem.

Throughout this section, we adopt a first-order system abstraction: UAVs move on a common flight layer, communication is predominantly LoS and later approximated by an effective disk model, peer-to-peer links are treated as symmetric, and the damage event occurs at $t_0$ before recovery begins. We further assume that a coarse mission-level reference (the virtual center used later in Section IV) and the last valid positions of failed nodes are available before severe fragmentation, e.g., through pre-failure topology maintenance, routing exchange, or shared flight-reference information \cite{mou2022resilient, chen2020sidr}. This information-availability mechanism is an operational assumption adopted to isolate the decentralized post-damage recovery problem, rather than an online estimation subsystem solved in the present paper. These assumptions isolate the topology restoration problem considered here; vertical maneuvers as an additional control degree of freedom, persistent stochastic link outages, online reference generation after the split, and sequential damage events are beyond the scope of the present system model.

\subsection{Communication Channel and Network Topology}
Consider a highly dynamic UAV swarm network (USNET) composed of $N$ homogeneous UAVs, represented by the node set $\mathcal{U}=\{u_1, u_2, \dots, u_N\}$. In many surveillance, search-and-rescue, and emergency-communication missions, UAVs coordinate on a common flight layer so that sensing coverage, search patterns, and communication footprints remain organized over the horizontal plane \cite{javed2024state, horyna2023decentralized}. Under this operating regime, horizontal separation is the dominant factor for pairwise connectivity, while moderate altitude deviations act mainly as a secondary perturbation to the link geometry; accordingly, planar coordination models are also common in CNS recovery methods such as HERO and SIDR \cite{mi2011hero, chen2020sidr}. We therefore formulate the spatial state of node $u_i \in \mathcal{U}$ at any discrete time step $t$ using a two-dimensional coordinate vector $\bm{p}_i(t) = [x_i(t), y_i(t)]^\top \in \mathbb{R}^2$, so the problem remains one of distributed connectivity restoration under incomplete post-damage network state information rather than pure motion coordination.

To characterize the air-to-air communication links while retaining a tractable topology model, we start from the Friis transmission equation with a fading term. Assuming a predominantly line-of-sight (LoS) channel, the received signal power at UAV $u_j$ from a transmitting UAV $u_i$ is given by
\begin{align}
    P_{ij}(t) = P_0 G_{tx} G_{rx} \left( \frac{\lambda_c}{4 \pi d_{ij}(t)} \right)^2 |h_0|^2,
    \label{eq:friis}
\end{align}
where $P_0$ is the transmission power, $G_{tx}$ and $G_{rx}$ are the antenna gains, $\lambda_c$ is the carrier wavelength, and $d_{ij}(t) = \|\bm{p}_i(t) - \bm{p}_j(t)\|$ denotes the Euclidean distance between the two UAVs. The term $|h_0|^2$ denotes the small-scale fading power gain. Rather than explicitly tracking instantaneous link outages, we adopt an equivalent communication radius $D_{comm}$ for a prescribed link budget and outage tolerance, and approximate the communication condition by $d_{ij}(t) \leq D_{comm}$. Consequently, bidirectional communication between $u_i$ and $u_j$ is feasible if and only if $d_{ij}(t) \leq D_{comm}$.

Since the UAVs are homogeneous and peer-to-peer links are assumed symmetric under the same transmission power and receiver sensitivity, the global physical topology of the USNET at time $t$ is modeled as an undirected graph $\mathcal{G}(t)=(\mathcal{U}, \mathcal{E}(t))$. The edge set $\mathcal{E}(t)$ is defined as $\mathcal{E}(t)=\{e_{ij}(t) \mid u_i, u_j \in \mathcal{U}, d_{ij}(t) \leq D_{comm}\}$. Accordingly, the physical neighborhood set of node $u_i$, which includes all reachable neighbors within its communication coverage, is formulated as
\begin{align}
    \mathcal{N}_i(t) = \{u_j \in \mathcal{U} \setminus \{u_i\} \mid e_{ij}(t) \in \mathcal{E}(t)\}.
    \label{eq:physical_neighbor}
\end{align}

\subsection{Mobility Model and Safety Constraints}

To characterize the controlled mobility of the UAV swarm, we adopt a discrete-time first-order kinematic model. Let $t \in \{0, 1, 2, \dots\}$ denote the discrete time index, and $\Delta t$ represent the constant duration of each control step. At each step $t$, the local controller of UAV $u_i \in \mathcal{U}$ outputs a velocity vector $\bm{v}_i(t) \in \mathbb{R}^2$, and the spatial position is updated as:
\begin{align}
    \bm{p}_i(t + \Delta t) = \bm{p}_i(t) + \bm{v}_i(t) \Delta t.
    \label{eq:kinematic}
\end{align}

Considering the physical maneuverability of the UAVs, the magnitude of the velocity vector is subject to a maximum speed constraint $v_{\max}$:
\begin{align}
    \|\bm{v}_i(t)\| \leq v_{\max}, \quad \forall u_i \in \mathcal{U}.
    \label{eq:v_max}
\end{align}

In addition to kinematic limits, the autonomous coordination must also avoid collisions. To prevent inter-UAV collisions during the connectivity restoration process, a minimum safety distance $D_{safe}$ is maintained between any pair of UAVs. This constraint is defined as:
\begin{align}
    \|\bm{p}_i(t) - \bm{p}_j(t)\| \geq D_{safe}, \quad \forall u_i, u_j \in \mathcal{U}, i \neq j.
    \label{eq:d_safe}
\end{align}

Under this mobility framework, the objective of the decentralized recovery algorithm is to determine an optimal velocity $\bm{v}_i(t)$ for each UAV, ensuring that both the speed constraint in \eqref{eq:v_max} and the safety constraint in \eqref{eq:d_safe} are satisfied throughout the mission duration.

\subsection{Damage Model and Network Connectivity}

We consider a scenario where the USNET undergoes massive node failures at initial time $t_0$ due to external hazards. Let $\mathcal{U}_{dmg} \subset \mathcal{U}$ denote the set of $N_{dmg}$ destroyed UAVs, and $\mathcal{U}_{act} = \mathcal{U} \setminus \mathcal{U}_{dmg}$ represent the set of $N_{act}$ remaining nodes. The post-damage topology is modeled as a remaining graph $\mathcal{G}_{act}(t) = (\mathcal{U}_{act}, \mathcal{E}_{act}(t))$, where the edge set $\mathcal{E}_{act}(t)$ consists of links between operational UAVs satisfying the distance constraint $d_{ij}(t) \leq D_{comm}$.

To mathematically evaluate the fragmentation of $\mathcal{G}_{act}(t)$, we define the adjacency matrix $\bm{A}_{act}(t) = [a_{ij}(t)] \in \{0, 1\}^{N_{act} \times N_{act}}$ for $i, j \in \{1, \dots, N_{act}\}$, where $a_{ij}(t) = 1$ if $e_{ij}(t) \in \mathcal{E}_{act}(t)$ and $a_{ij}(t) = 0$ otherwise. The degree matrix $\bm{D}_{act}(t)$ is a diagonal matrix with entries $D_{ii} = \sum_{j=1}^{N_{act}} a_{ij}$. The Laplacian matrix of the remaining graph is then formulated as:
\begin{equation}
    \bm{L}_{act}(t) = \bm{D}_{act}(t) - \bm{A}_{act}(t).
    \label{eq:laplacian}
\end{equation}

The connectivity properties of $\mathcal{G}_{act}(t)$ are characterized by the eigenvalues of $\bm{L}_{act}(t)$, denoted as $0 = \lambda_1(t) \leq \lambda_2(t) \leq \dots \leq \lambda_{N_{act}}(t)$. The graph is connected if and only if the second smallest eigenvalue (i.e., the Fiedler value) satisfies $\lambda_2(t) > 0$ \cite{mohar1991laplacian}. In the event of a Communication Network Split (CNS), the graph decomposes into $N_s(t)$ disjoint sub-networks, where $N_s(t)$ equals the multiplicity of the zero eigenvalue of $\bm{L}_{act}(t)$. Therefore, we focus on massive damage scenarios where $N_s(t_0) > 1$. The recovery objective is to control the velocity $\bm{v}_i(t)$ of each remaining UAV $u_i \in \mathcal{U}_{act}$ such that the disjoint sub-networks merge, eventually achieving $\lambda_2(t_0 + T_{res}) > 0$ within a finite duration $T_{res}$.

\subsection{Problem Formulation for Decentralized Recovery}

The primary objective of the resilient coordination is to restore the global network connectivity of the fragmented USNET within a minimum time duration $T_{res}$. Unlike centralized approaches that rely on a global state, we consider a decentralized execution paradigm where each remaining UAV $u_i \in \mathcal{U}_{act}$ determines its control action from localized observations together with the shared reference established before fragmentation.

At each time step $t$, a UAV $u_i$ perceives its local environment through a limited observation space $\mathcal{O}_i(t)$, which is derived from its physical neighborhood $\mathcal{N}_i(t)$, nearby failed-node cues, and the pre-shared virtual center. The decentralized control policy, parameterized by $\theta$ and denoted as $\pi_\theta$, maps the local observation to a continuous velocity vector:
\begin{equation}
    \bm{v}_i(t) = \pi_\theta(\mathcal{O}_i(t)), \quad \forall u_i \in \mathcal{U}_{act}.
\end{equation}

Mathematically, the decentralized connectivity restoration problem is formulated as:
\begin{align}
    (\text{P1}): \quad \min_{\pi_\theta} \quad & \mathbb{E}[T_{\text{res}}] \\
    \text{s.t.} \quad & \lambda_2(t_0 + T_{\text{res}}) > 0, \tag{C1} \label{con:connectivity} \\
    & \|\bm{v}_i(t)\| \leq v_{\max}, \quad \forall u_i, t, \tag{C2} \label{con:v_max} \\
    & \|\bm{p}_i(t) - \bm{p}_j(t)\| \geq D_{safe}, \quad \forall i \neq j, t, \tag{C3} \label{con:d_safe} \\
    & \mathcal{O}_i(t) \subseteq \{\bm{p}_j(t), \bm{v}_j(t) \mid u_j \in \mathcal{N}_i(t) \cup \{u_i\}\} \nonumber\\
    &\quad \cup \{\bm{p}_k(t) \mid u_k \in \mathcal{U}_{dmg},\, d_{ik}(t)\le D_{comm}\} \nonumber\\
    &\quad \cup \{\bm{p}_{center}\}. \tag{C4} \label{con:local_obs}
\end{align}

Constraints \eqref{con:connectivity}-\eqref{con:d_safe} enforce global connectivity restoration, kinematic limits, and collision avoidance, respectively. Constraint \eqref{con:local_obs} restricts the policy execution to a strictly decentralized manner. Problem P1 therefore defines the ideal constrained control objective. Since directly solving P1 over decentralized continuous policies is intractable, Section IV introduces a reward-based learning surrogate that is designed to approximate the minimization of $T_{res}$ while respecting the same physical and informational constraints during policy learning and execution.

\section{Physics-informed Graph Adversarial Imitation Learning Approach}
This section first introduces the overall architecture of the proposed PhyGAIL algorithm. Then the second subsection details the local scale-invariant graph perception mechanism. The physics-informed graph neural networks and imitation learning strategy are presented in the third and fourth subsections, respectively. The fifth subsection analyzes structural properties of the resulting architecture, and the sixth subsection derives the algorithm's complexity.

\begin{figure*}[!t]
\centerline{\includegraphics[width=.95\linewidth]{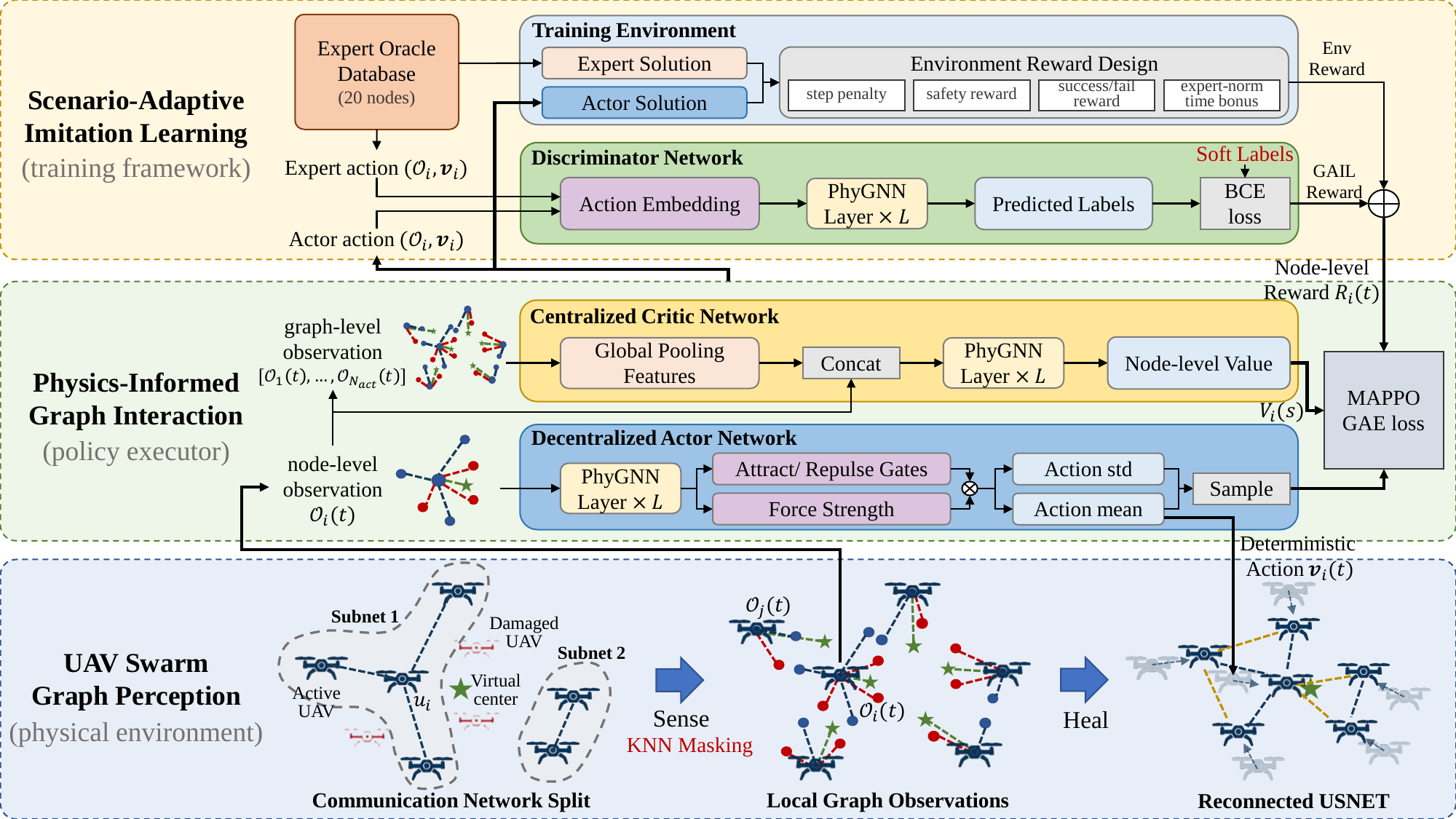}}
\captionsetup{justification=raggedright, singlelinecheck=false}
\caption{Overall training and execution architecture of PhyGAIL.}
\label{fig:framework}
\end{figure*}

\subsection{Overall Framework of PhyGAIL}

Figure~\ref{fig:framework} shows the overall architecture of PhyGAIL under centralized training with decentralized execution. The framework contains three coupled modules: UAV swarm graph perception, physics-informed graph interaction, and scenario-adaptive imitation learning. These modules correspond to three stages of graph-MARL adaptation in large damaged swarms: bounded graph construction defines a scale-stable local decision domain, physics-gated interaction modeling resolves attraction-repulsion ambiguity inside that domain, and expert-guided imitation stabilizes learning under sparse and imbalanced recovery feedback without rebuilding a global post-damage network view before recovery.

\textbf{UAV Swarm Graph Perception:}
The bottom layer represents the physical environment and constructs the bounded observation $\mathcal{O}_i(t)$ for each active UAV. Within the communication range $D_{comm}$, the local graph contains active UAVs, damaged UAVs, and a virtual center $\bm{p}_{center}$, where the virtual center provides a pre-shared merging reference and damaged-node coordinates provide residual fault cues \cite{mou2022resilient, chen2020sidr}. A K-nearest-neighbor masking mechanism then limits the number of active and damaged neighbors, so the local decision domain remains scale-stable. This shared-reference acquisition is treated as part of the recovery setting, not as an additional online module of PhyGAIL.

\textbf{Physics-Informed Graph Interaction:}
The middle layer is the decentralized policy executor. For each active UAV $u_i$, the actor maps $\mathcal{O}_i(t)$ to a continuous velocity command $\bm{v}_i(t)$ with an $L$-layer physics-informed graph neural network (PhyGNN). Here, ``physics-informed'' means that relative position and velocity cues are decomposed into attraction and repulsion gates together with a learnable force strength, rather than being left to generic graph aggregation. During training, a centralized critic uses masked global pooling and a separate PhyGNN encoder to estimate the node-level value $V_i(s)$.

\textbf{Scenario-Adaptive Imitation Learning:}
The top layer provides the offline training mechanism. An expert trajectory database and a discriminator network generate dense GAIL rewards from local observations and actions, and these rewards are combined with the environment reward to form the final node-level reward $R_i(t)$. The environment reward includes the step penalty, safety-related terms, terminal feedback, and an expert-normalized time bonus. The actor and critic are then optimized with MAPPO and GAE.

Together, these modules form a unified decentralized recovery framework. During offline training, the actor, centralized critic, expert trajectory database, and discriminator are all active; during test-time execution, each UAV keeps only the local graph perception module and the trained actor with its PhyGNN encoder, so deployment does not require retraining or centralized supervision. In this sense, bounded graph construction makes large-scale deployment feasible, the gated interaction model makes the learned behavior physically meaningful, and expert-guided imitation makes the policy trainable under severe topology fragmentation.

\subsection{Bounded Heterogeneous Graph Perception}

In a fully decentralized swarm, each UAV can access only local information from nearby agents. After a CNS event, surviving neighbors alone are often insufficient to reveal the damaged region or the direction needed for reconnection. We therefore represent the local observation of each active UAV as a bounded heterogeneous perception graph. Figure~\ref{fig:graph_perception} shows the resulting local graph, which augments local neighbor information with both fault-related cues and a shared spatial reference.

\subsubsection{Heterogeneous Node Representation}

The local graph is defined as $\mathcal{G}_{local}=(\mathcal{U}_{local},\mathcal{E}_{local})$, where the node set $\mathcal{U}_{local}$ contains three categories indicated by a type variable $c_i\in\{0,1,2\}$:

\begin{itemize}
    \item \textbf{Active Nodes (Type 0):} operational UAVs in the local sub-network. Each node encodes the UAV kinematic state, including position and velocity, for coordinated motion and collision avoidance.
    \item \textbf{Damaged Nodes (Type 1):} incapacitated UAVs within communication range ($d_{ij}\leq D_{comm}$). Their velocities are masked to zero, while their coordinates serve as spatial anchors that outline the damaged region.
    \item \textbf{Virtual Center (Type 2):} a predefined reference node $u_c$ at position $\bm{p}_c$, used to provide isolated sub-networks with a shared merging direction. This reference is assumed to be known before severe fragmentation through mission planning or pre-failure coordination \cite{chen2020sidr}.
\end{itemize}

The edge set also contains heterogeneous interaction types: active--active communication links, damaged--active observation links, and center--active guidance links. This abstraction does not assume that a disconnected subnet can reconstruct the entire post-damage global graph online. Instead, it assumes only a limited prior reference and residual fault cues, whereas complete whole-swarm topology access remains the domain of centralized methods such as CR-MGC \cite{mou2022resilient}.

\begin{figure}[!t]
\centerline{\includegraphics[width=1.0\linewidth]{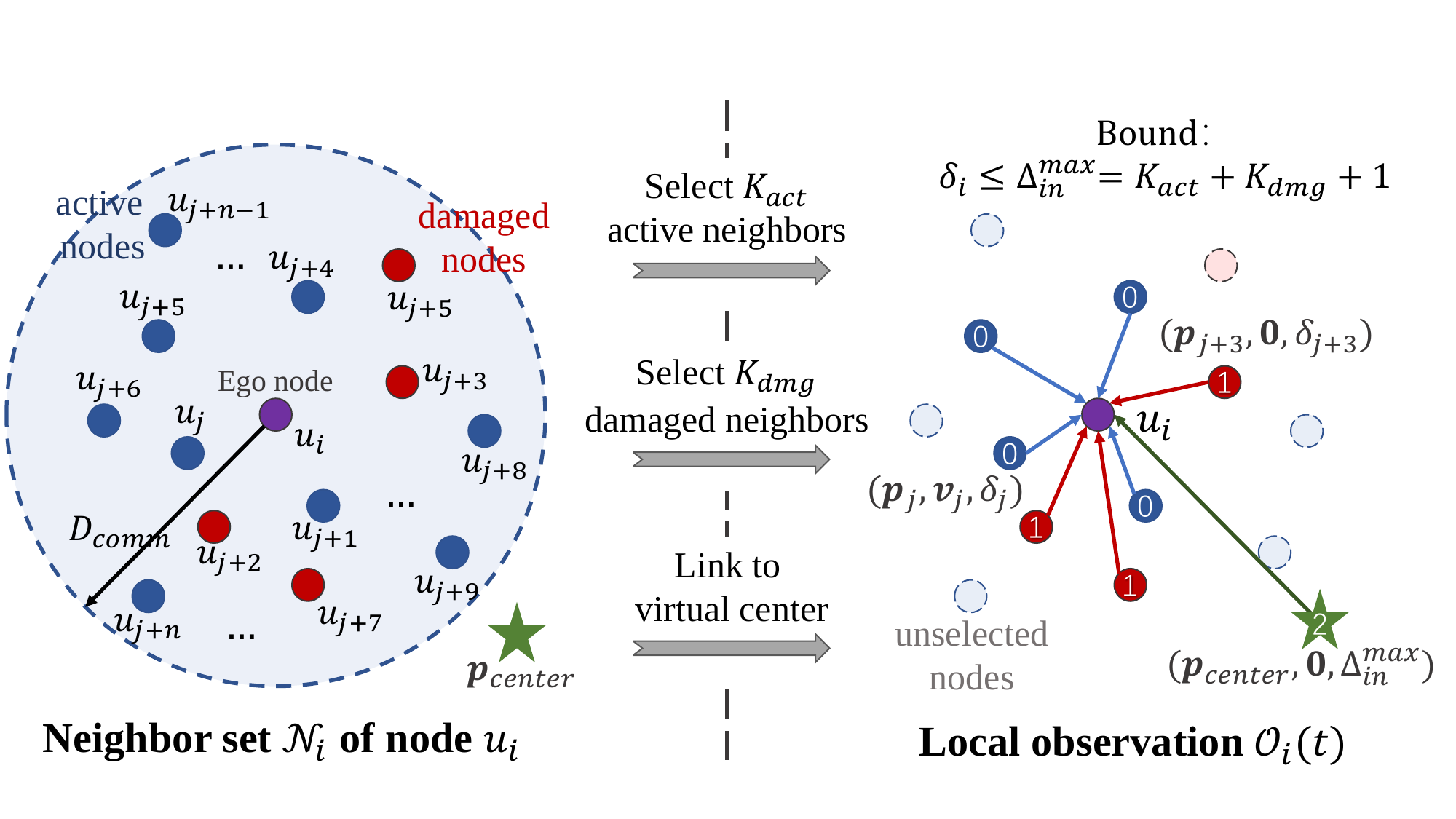}}
\captionsetup{justification=raggedright, singlelinecheck=false}
\caption{Construction of the bounded heterogeneous local observation graph.}
\label{fig:graph_perception}
\end{figure}

To keep the local receptive field bounded regardless of the global swarm size $N$, we use a directed K-NN \cite{dong2011efficient} masking mechanism. Although the communication radius $D_{comm}$ limits interaction distance, the number of neighbors within this radius can still grow with swarm density. Figure~\ref{fig:graph_perception} illustrates how K-NN masking is used to build a bounded computational topology.

For each active node $u_i$, we first identify the set of operational neighbors within communication range:
\begin{equation}
\mathcal{N}_i^{comm}=\{u_j\in\mathcal{U}_{act}\setminus\{u_i\}\mid d_{ij}\le D_{comm}\}.
\end{equation}
Node $u_i$ then selects at most $K_{act}$ closest active neighbors from $\mathcal{N}_i^{comm}$ to form the directed subset $\mathcal{N}_i^{K_{act}}$. Since each node gathers states only from the neighbors it selects, a directed edge $e_{j\to i}$ is established if $u_j\in\mathcal{N}_i^{K_{act}}$:
\begin{align}
e_{j\to i}\in\mathcal{E}_{local}\iff u_j\in\mathcal{N}_i^{K_{act}}.
\end{align}

This masking keeps the maximum in-degree of any node bounded and decouples the local observation dimension from the global swarm scale. To avoid unbounded feature growth under catastrophic failures, each active UAV also attends to at most $K_{dmg}$ nearest damaged neighbors within $D_{comm}$, forming the damaged-neighbor set $\mathcal{N}_i^{K_{dmg}}$.

Including the virtual center node $u_c$, the masked local neighbor set of $u_i$ is
\begin{equation}
\mathcal{N}_i^{local}=\mathcal{N}_i^{K_{act}}\cup\mathcal{N}_i^{K_{dmg}}\cup\{u_c\}.
\end{equation}
The maximum in-degree is therefore bounded by
\begin{align}
\Delta_{in}^{\max}=K_{act}+K_{dmg}+1.
\label{eq:max_degree}
\end{align}

\subsubsection{Local Observation Encoding}

Based on the bounded heterogeneous graph, we construct the continuous observation used by the policy network. To improve robustness to spatial translation and varying swarm density, the observation uses relative and normalized kinematic states rather than absolute coordinates.

Each node position is represented relative to the virtual center $\bm{p}_{center}$ and normalized by the geographical scale factor $W_{scale}$:
\begin{align}
\tilde{\bm{p}}_i(t)=\frac{\bm{p}_i(t)-\bm{p}_{center}}{W_{scale}}.
\end{align}
Similarly, the velocity is normalized by the mechanical speed limit $v_{\max}$:
\begin{align}
\tilde{\bm{v}}_i(t)=\frac{\bm{v}_i(t)}{v_{\max}}.
\end{align}
For damaged nodes and the virtual center, i.e., $c_i\in\{1,2\}$, the velocity is fixed to zero.

We also include the structural degree $\delta_i(t)$ to reflect local congestion. To keep this feature in a bounded range, we apply logarithmic normalization:
\begin{equation}
\tilde{\delta}_i(t)=\frac{\ln(1+\delta_i(t))}{\ln(1+\Delta_{in}^{\max})},
\end{equation}
where $\Delta_{in}^{\max}$ is defined in (\ref{eq:max_degree}). The raw spatial-kinematic feature vector is then
\begin{equation}
\bm{x}_i(t)=[\tilde{\bm{p}}_i(t)^\top,\tilde{\bm{v}}_i(t)^\top,\tilde{\delta}_i(t)]^\top.
\end{equation}

Finally, the categorical node type $c_i$ is mapped into a learnable embedding space through $\phi_{node}(\cdot)$. The final node observation is
\begin{equation}
\mathcal{O}_i(t)=[\bm{x}_i(t)^\top,\phi_{node}(c_i)^\top]^\top.
\end{equation}
As a result, the policy input dimension scales as $\mathcal{O}(\Delta_{in}^{\max})$ rather than with the global swarm size $N$, which supports zero-shot transfer to larger swarms.

\subsection{Physics-Informed Graph Interactions}

Having defined the bounded local observation graph, we now describe the core neural architecture of PhyGAIL. This part contains two components: a physics-gated message passing mechanism for local interaction modeling, and MAPPO-based actor--critic networks that map local observations to continuous velocity actions and state-value estimates.

\begin{figure}[!t]
\centerline{\includegraphics[width=1.0\linewidth]{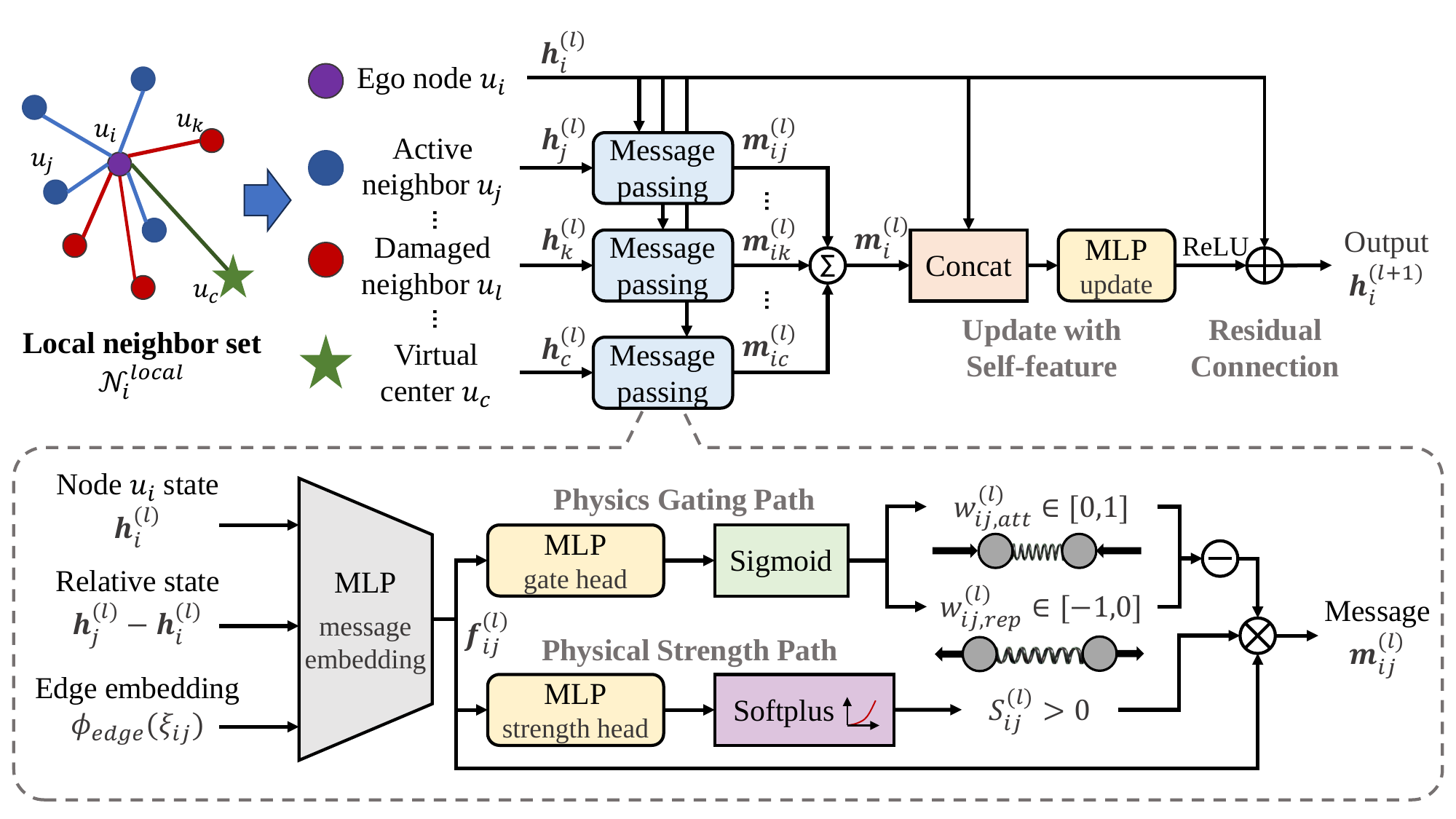}}
\captionsetup{justification=raggedright, singlelinecheck=false}
\caption{Physics-gated message passing in the proposed PhyGNN layer.}
\label{fig:phygnn}
\end{figure}

\subsubsection{Physics-Gated Message Passing}

Figure~\ref{fig:phygnn} summarizes the proposed physics-gated message passing mechanism. Conventional graph neural networks do not explicitly model the attraction, repulsion, and collision-related interaction semantics required in continuous-space multi-UAV coordination. We therefore introduce a message passing rule that decomposes each directed edge interaction into attraction and repulsion gates together with a learnable force strength.

Let $\bm{h}_i^{(l)}$ denote the hidden state of node $u_i$ at layer $l$. We initialize the hidden state with the localized graph observation, i.e.,
\begin{equation}
    \bm{h}_i^{(0)}=\mathcal{O}_i(t).
\end{equation}
For a directed edge $e_{j\to i}$ with discrete type $\xi_{ij}\in\{0,1,2\}$, we construct the interaction feature $\bm{f}_{ij}^{(l)}$ from the receiver state, the relative state, and the semantic edge embedding $\phi_{edge}(\xi_{ij})$:
\begin{equation}
    \bm{f}_{ij}^{(l)}=\text{MLP}_{msg}^{(l)}\left([\bm{h}_i^{(l)\top},(\bm{h}_j^{(l)}-\bm{h}_i^{(l)})^\top,\phi_{edge}(\xi_{ij})^\top]^\top\right).
\end{equation}
For edges involving active or damaged nodes, this feature encodes relative kinematics explicitly. For virtual-center edges ($\xi_{ij}=2$), the relative velocity is zero, so the message acts only as a spatial anchor.

We use a dual-gating mechanism to separate different interaction modes. The feature $\bm{f}_{ij}^{(l)}$ passes through a gating network with a Sigmoid activation, producing an attraction gate $w_{ij,att}^{(l)}$ and a repulsion gate $w_{ij,rep}^{(l)}\in[0,1]$. In parallel, a Softplus-activated network outputs the interaction force magnitude $S_{ij}^{(l)}>0$:
\begin{align}
    [w_{ij,att}^{(l)},w_{ij,rep}^{(l)}]^\top &= \text{Sigmoid}\left(\text{MLP}_{gate}^{(l)}(\bm{f}_{ij}^{(l)})\right), \\
    S_{ij}^{(l)} &= \text{Softplus}\left(\text{MLP}_{strength}^{(l)}(\bm{f}_{ij}^{(l)})\right).
\end{align}

The directed physical message is
\begin{equation}
    \bm{m}_{j\to i}^{(l)}=\bm{f}_{ij}^{(l)}\cdot S_{ij}^{(l)}\cdot\left(w_{ij,att}^{(l)}-w_{ij,rep}^{(l)}\right).
\end{equation}

\subsubsection{PhyGNN Layer Forward Propagation}

Figure~\ref{fig:phygnn} also shows the forward pass of a single PhyGNN layer. We aggregate incoming physical messages from the local computational neighborhood $\mathcal{N}_i^{local}$ according to the force superposition principle:
\begin{equation}
    \bm{m}_i^{(l)}=\sum_{u_j\in\mathcal{N}_i^{local}}\bm{m}_{j\to i}^{(l)}.
\end{equation}

We then update the node representation with a nonlinear transformation and a residual connection:
\begin{equation}
    \bm{h}_i^{(l+1)}=\text{ReLU}\left(\text{MLP}_{update}^{(l)}([\bm{h}_i^{(l)\top},\bm{m}_i^{(l)\top}]^\top)\right)+\bm{h}_i^{(l)}.
\end{equation}
The residual term helps stabilize gradient propagation during adversarial training.

Stacking $L$ such PhyGNN layers enables the network to capture multi-hop spatial interactions within the swarm. The final embedding $\bm{h}_i^{(L)}$ serves as the feature representation for the task-specific actor and critic heads.

\subsubsection{MAPPO-Based Actor-Critic Networks}

We use separate PhyGNN encoders in the actor and critic networks of the MAPPO framework \cite{yu2022surprising}. This avoids parameter sharing between policy optimization and value estimation.

\textbf{Actor Network:}  
The decentralized actor $\pi_\theta(a_i\mid\mathcal{O}_i)$ maps local observations to continuous velocity commands. For each active node $u_i$, two parallel MLPs output the mean $\bm{\mu}_i$ and log-standard deviation of a Gaussian policy. During training, we sample actions from this distribution for exploration. During execution, we squash the mean with a hyperbolic tangent and scale it by the mechanical limit:
\begin{equation}
    \bm{v}_i(t)=v_{\max}\cdot\tanh(\bm{\mu}_i).
\end{equation}

\textbf{Centralized Critic Network:}  
Under the CTDE paradigm, the centralized critic $V_\phi$ receives the PhyGNN embeddings $\bm{h}_i^{(L)}$ of all nodes in the global graph. We then apply masked dual pooling over the active-node subset $\mathcal{U}_{act}$ to aggregate global features:
\begin{equation}
    \bm{h}_{global}=[\text{mean}(\{\bm{h}_i^{(L)}\}_{i\in\mathcal{U}_{act}})\parallel \max(\{\bm{h}_i^{(L)}\}_{i\in\mathcal{U}_{act}})].
\end{equation}
We concatenate this pooled feature with the ego embedding to form an agent-specific representation,
\begin{equation}
    \bm{h}_i^{combined}=[\bm{h}_i^\top,\bm{h}_{global}^\top]^\top,
\end{equation}
and use a value head to predict the agent-specific state value $V_i(\bm{s})$. During training, we optimize the critic by minimizing the squared temporal-difference error with respect to the GAE targets. This masked architecture reduces the interference of static nodes and improves value estimation stability under different damage levels.

\subsection{Scenario-Adaptive Imitation Learning Strategy}

Sparse environmental rewards alone are insufficient for training the PhyGNN policy across fragmented CNS scenarios with different damage levels. We therefore introduce a scenario-adaptive adversarial imitation learning strategy under the CTDE paradigm. This part contains three components: expert demonstration generation, step-level adversarial imitation rewards, and an expert-normalized temporal baseline.

\subsubsection{Expert Demonstration Generation}

High-quality expert demonstrations guide the policy toward collision-safe and time-efficient connectivity restoration. Since no single algorithm consistently performs best over all fragmented topologies, we build the expert dataset offline by running multiple baseline algorithms, including both centralized optimization methods and decentralized heuristics, over simulated damage scenarios. For each topology, we select the trajectory with the minimum restoration time as the expert reference. An expert trajectory is defined as
\begin{equation}
    \tau_E=\{(\mathcal{O}_i^{(E)}(t),\bm{v}_i^{(E)}(t))\}_{t=0}^{T_E},
\end{equation}
where $T_E$ denotes the task completion time. Because both expert trajectories and policy rollouts are generated in the same simulator, their observation spaces are consistent, which avoids distribution mismatch during discriminator training.
This composite expert pool is used to broaden offline supervision coverage across fragmented scenarios, rather than to transfer centralized test-time information into the decentralized actor.

This expert pool is used only during offline training. Even when a selected trajectory comes from a centralized baseline with whole-topology access, that global information is never queried by the learned actor during decentralized execution. The centralized trajectories therefore act only as privileged supervision for expert selection, not as test-time side information for PhyGAIL.

We also use spatial data augmentation to improve sample efficiency and reduce overfitting to absolute geographic coordinates. Specifically, we augment the positions and velocities in each expert trajectory with the dihedral symmetry group $D_4$ \cite{cohen2016group}. By applying $90^\circ$ rotations and axis reflections, each original trajectory yields eight geometrically isomorphic variants. The resulting expert buffer $\mathcal{B}_E$ encourages the discriminator to focus on relative topology rather than absolute orientation, which improves zero-shot generalization.

\subsubsection{Step-Level Adversarial Imitation Reward}

In highly fragmented topologies, successful subnetwork merging occurs rarely, so the environmental reward is extremely sparse. We therefore use Generative Adversarial Imitation Learning (GAIL) \cite{ho2016generative} to provide dense step-level guidance through a node-level discriminator $D_\omega$.

Instead of operating on the whole swarm, the discriminator evaluates localized state-action pairs. For an active node $u_i$, we use an independent $L$-layer PhyGNN encoder to extract the state embedding $\bm{h}_i^{D}$ from the localized observation $\mathcal{O}_i$. We then normalize the executed velocity command by the mechanical limit and map it to an action embedding with a multi-layer perceptron:
\begin{equation}
    \bm{z}_i^{A}=\text{MLP}_{act}\left(\frac{\bm{v}_i}{v_{\max}}\right).
\end{equation}
We concatenate the state and action embeddings and feed them into a discriminator head with LeakyReLU activations:
\begin{equation}
    D_\omega(\mathcal{O}_i,\bm{v}_i)=\text{LeakyReLU}\left(\text{MLP}_{head}([\bm{h}_i^{D\top},\bm{z}_i^{A\top}]^\top)\right).
\end{equation}
Here, $D_\omega(\cdot)\in(0,1)$ denotes the probability that the state-action pair comes from the expert policy.

To reduce discriminator overconfidence, we apply one-sided label smoothing \cite{szegedy2016rethinking}. The discriminator minimizes the BCE loss with soft labels, where expert trajectories use target probability $\lambda_E=0.9$ and actor-generated trajectories use $\lambda_A=0.1$:
\begin{align}
    \mathcal{L}_D(\omega) &= - \mathbb{E}_{(\mathcal{O}_i,\bm{v}_i)\sim\mathcal{B}_E}\left[\lambda_E \log D_\omega(\mathcal{O}_i,\bm{v}_i)\right] \nonumber \\
    &\quad - \mathbb{E}_{(\mathcal{O}_i,\bm{v}_i)\sim\pi_\theta}\left[(1-\lambda_A)\log(1-D_\omega(\mathcal{O}_i,\bm{v}_i))\right].
\end{align}
Here, $\mathcal{B}_E$ is the mini-batch sampled from the expert database.

We convert the stabilized discriminator output into a step-level imitation reward. For numerical stability, we clip the discriminator output with a small constant $\epsilon=10^{-8}$:
\begin{equation}
    \tilde{D}_\omega=\text{clip}(D_\omega(\mathcal{O}_i(t),\bm{v}_i(t)),\epsilon,1-\epsilon).
\end{equation}
The dense imitation reward is then
\begin{equation}
    r_i^{GAIL}(t)=-\log(1-\tilde{D}_\omega).
\end{equation}
Maximizing this reward encourages the policy to match the collision-safe merging behavior contained in the expert demonstrations.

\subsubsection{Expert-Normalized Temporal Baseline}

Constant step penalties are often used to encourage fast task completion. In CNS restoration, however, recovery horizons vary significantly with the damage level. Mild fragmentation may be resolved in tens of steps, whereas severe fragmentation may require several hundred steps. A fixed step penalty therefore biases the value function toward short-horizon tasks and can destabilize learning.

We reduce this effect with an expert-normalized temporal baseline. For a given scenario, let $T_E$ denote the completion steps of the expert algorithm. We define the dynamic speed bonus factor as
\begin{equation}
    \eta(t)=\min\left(\exp\left(\lambda\left(1-\frac{t}{\alpha T_E}\right)\right),\eta_{\max}\right).
\end{equation}
Here, $\alpha\geq1$ is a tolerance margin relative to the expert time, $\lambda$ controls the exponential decay, and $\eta_{\max}$ caps the maximum bonus.

When connectivity is successfully restored, i.e., all subnetworks merge into one connected component with $N_s=1$, the active UAVs receive the terminal reward
\begin{equation}
    r^{success}(t)=r^{base}+r^{time}\cdot\eta(t).
\end{equation}
Here, $r^{base}$ is the base success reward and $r^{time}$ is the speed bonus pool. If the learned policy matches the expert pace, it receives the nominal bonus; if it reconnects faster, the bonus increases; if it is slower, the bonus decays smoothly. This normalization keeps reward scaling more consistent across scenarios with different damage severities.

\subsubsection{Reward Design and Training}

For each active node $u_i$, the total step-level reward combines the adversarial imitation term and the environmental reward:
\begin{equation}
    R_i(t)=w_{IL}\cdot r_i^{GAIL}(t)+r_i^{env}(t).
\end{equation}
Here, $w_{IL}$ is the imitation coefficient. We decompose the environmental reward into three parts:
\begin{equation}
    r_i^{env}(t)=r^{step}+r_i^{safe}(t)+r^{term}(t).
\end{equation}

The urgency penalty $r^{step}<0$ is a small constant that discourages loitering. To enforce physical safety and reduce local congestion, we define the safety penalty as
\begin{equation}
    r_i^{safe}(t)=-\omega_{col}\sum_{u_j\in\mathcal{N}_i^{comm}}\max(0,D_{safe}-d_{ij}),
\end{equation}
where $d_{ij}$ is the physical distance between nodes and $\omega_{col}$ is the scaling weight. This term discourages collisions directly.

The terminal reward incorporates the expert-normalized success reward when restoration succeeds. If the maximum episode length $T_{\max}$ is reached without convergence, we assign a failure penalty proportional to the number of remaining disconnected components:
\begin{equation}
    r^{term}(t)=
    \begin{cases}
    r^{success}(t), & \text{if } N_s = 1, \\
    -P^{fail}\cdot N_s, & \text{if } t = T_{\max} \text{ and } N_s > 1, \\
    0, & \text{otherwise}.
    \end{cases}
\end{equation}

Finally, we clip the synthesized reward $R_i(t)$ to the bounded range $[-R_{\max},R_{\max}]$ to stabilize critic training and reduce variance in value estimation. The overall training procedure is summarized in Algorithm~\ref{alg:phygail}.

\begin{algorithm}[tbp]
\caption{CTDE Training Procedure of PhyGAIL}
\label{alg:phygail}
\begin{algorithmic}[1]
\renewcommand{\algorithmicrequire}{\textbf{Input:}}
\renewcommand{\algorithmicensure}{\textbf{Output:}}
\REQUIRE Augmented expert buffer $\mathcal{B}_E$, maximum epochs $E_{\max}$, rollout horizon $T$, batch size $B$.
\ENSURE Trained actor $\pi_\theta$, critic $V_\phi$, and discriminator $D_\omega$.
\FOR{$epoch = 1,2,\dots,E_{\max}$}
    \STATE Reset the parallel environments and initialize the rollout buffer $\mathcal{B}_{rollout}$.
    \FOR{$t = 1,2,\dots,T$}
        \STATE Construct localized observations $\mathcal{O}_i(t)$ and apply active-node masking.
        \STATE Sample actions $\bm{v}_i \sim \pi_\theta(\cdot|\mathcal{O}_i(t))$ for all active UAVs.
        \STATE Execute actions, observe next states, and collect environmental rewards $r_i^{env}(t)$.
        \STATE Store the resulting transitions in $\mathcal{B}_{rollout}$.
    \ENDFOR
    \STATE Sample expert state-action pairs from $\mathcal{B}_E$ and policy rollouts from $\mathcal{B}_{rollout}$.
    \STATE Update the discriminator $D_\omega$ with BCE loss and label smoothing.
    \STATE Compute the imitation reward $r_i^{GAIL}(t)$ for each active node.
    \STATE Compute the final reward $R_i(t)=w_{IL}r_i^{GAIL}(t)+r_i^{env}(t)$ and clip to $[-R_{\max},R_{\max}]$.
    \STATE Compute GAE advantages and value targets.
    \STATE Update the actor parameters $\theta$ and critic parameters $\phi$ with MAPPO.
\ENDFOR
\end{algorithmic}
\end{algorithm}

\subsection{Structural Property Analysis}

This section examines several structural properties of PhyGAIL. Because decentralized multi-agent interactions are nonlinear and discrete, the following results are not performance guarantees for recovery optimality, convergence speed, or zero-shot success. We instead focus on three properties that matter directly for large-scale deployment: the bounded spectral norm of the local perception graph, the bounded pseudo-force generated by physics-gated message passing, and the controlled variance of the expert-normalized terminal reward.

\subsubsection{Bounded Spectral Norm of Scale-Invariant Perception}

In large-scale MARL, graph feature magnitudes may grow uncontrollably when a policy trained on small swarms is deployed on much larger systems. Standard unmasked communication graphs can indeed have unbounded degree. In PhyGAIL, by contrast, the perception graph is bounded by K-NN truncation and physical spatial constraints.

We assume that all UAVs maintain a minimum safety distance $D_{safe} > 0$ during flight and that communication links exist only within radius $D_{comm}$.

\begin{proposition}
\label{prop:bounded_spectral_norm}
\textbf{Bounded Spectral Norm of Directed Perception Graphs.}
Let $\tilde{\bm{A}}$ denote the asymmetric adjacency matrix of the directed K-NN perception graph $\mathcal{G}_{local}$, where directed edges represent message-passing flow. Under algorithmic truncation and physical constraints, the spectral norm (induced $2$-norm) of $\tilde{\bm{A}}$ is bounded by a constant independent of the global swarm size $N$:
\begin{equation}
    \|\tilde{\bm{A}}\|_2 \le \sqrt{\Delta_{in}^{\max} \cdot \Delta_{out}^{\max}},
\end{equation}
where the maximum in-degree is bounded algorithmically by $\Delta_{in}^{\max} = K_{act} + K_{dmg} + 1$, and the maximum out-degree satisfies $\Delta_{out}^{\max} \le C_{pack}$. Here, $C_{pack} = \mathcal{O}((D_{comm}/D_{safe})^d)$ is a conservative sphere-packing bound \cite{conway2013sphere} on the maximum number of coexisting nodes within the communication region in a $d$-dimensional space.
\end{proposition}
\begin{IEEEproof}
See Appendix~\ref{app:proof_spectral_norm}.
\end{IEEEproof}

This result formalizes the role of bounded graph construction: even when the global swarm size grows, the local aggregation operator seen by each UAV remains numerically controlled. Thus, the amplification factor of graph propagation stays bounded as the swarm grows, which provides a structural explanation for the zero-shot scale generalization observed in the full framework.

\subsubsection{Bounded Dynamics of Physics-Gated Message Passing}

We next analyze the boundedness of the pseudo-force generated by the physics-gated message passing mechanism. This property is directly related to stable mobility control.

\begin{proposition}
\label{prop:bounded_force}
\textbf{Bounded Pseudo-Force Generation.}
Suppose the local observation $\mathcal{O}_i(t)$ lies in a compact feature space $\mathcal{X}$. Then the aggregated pseudo-force $\bm{m}_i$ produced by the PhyGNN encoder is bounded.
\end{proposition}
\begin{IEEEproof}
See Appendix~\ref{app:proof_bounded_force}.
\end{IEEEproof}

This bound complements Proposition~\ref{prop:bounded_spectral_norm}: after bounding how much information can be aggregated, we also bound the magnitude of the resulting interaction response. As a consequence, the policy network cannot generate arbitrarily large internal interaction signals, which helps keep the coordination dynamics numerically stable.

\subsubsection{Variance Decoupling via Expert-Normalized Baseline}

Long-horizon reinforcement learning tasks often have high return variance because episode length changes from scenario to scenario. Here we examine the terminal success signal induced by the expert-normalized reward.

\begin{proposition}
\label{prop:variance_decoupling}
\textbf{Variance Control of the Terminal Success Signal.}
Let the terminal success reward be
\begin{equation}
    r_{term}^{+} = r_{base} + r_{speed}\eta(t),
\end{equation}
where the normalized speed bonus factor is clipped such that $\eta(t)\in[0,\eta_{\max}]$. Then
\begin{equation}
    \mathrm{Var}(r_{term}^{+}) \le \frac{1}{4}(r_{speed}\eta_{\max})^2.
\end{equation}
\end{proposition}
\begin{IEEEproof}
See Appendix~\ref{app:proof_variance_decoupling}.
\end{IEEEproof}

This result clarifies why expert normalization is useful beyond heuristic reward shaping: it explicitly limits the spread of the positive terminal signal across scenarios with different completion times. The bound applies only to the terminal success signal. Step penalties may still introduce variance through stochastic episode length, but the positive success reward itself remains bounded. Hence, value estimation is less sensitive to large differences in task horizon.

\subsection{Complexity Analysis}

We analyze the communication, computational, and memory costs of PhyGAIL during decentralized execution. For practical UAV swarm deployment, each agent should keep its onboard resource consumption bounded and independent of the global swarm size $N$.

During execution, each active UAV broadcasts only a fixed-size local state, including relative position, velocity, and current degree, to neighbors within the communication range $D_{comm}$. Proposition~\ref{prop:bounded_spectral_norm} bounds the number of active UAVs inside this region by the geometric packing constant $C_{pack}$, so the per-step communication overhead of each UAV remains $\mathcal{O}(C_{pack})=\mathcal{O}(1)$ with respect to $N$.

The onboard computation has two bounded parts: local graph construction and PhyGNN inference. In graph construction, each UAV selects the $K_{act}$ nearest active neighbors and the $K_{dmg}$ nearest damaged neighbors from a bounded candidate set, yielding worst-case complexity $\mathcal{O}(C_{pack}\log K_{act})$. In inference, the ego node receives at most $\Delta_{in}^{\max}\leq K_{act}+K_{dmg}+1$ incoming messages, so an $L$-layer PhyGNN with hidden dimension $d_h$ requires
\begin{equation}
    \mathcal{O}\left(C_{pack}\log K_{act}+L\cdot \Delta_{in}^{\max}\cdot d_h^2\right).
\end{equation}
All terms are constants or bounded hyperparameters, so the per-agent computational complexity is also $\mathcal{O}(1)$ with respect to $N$. 

The memory footprint is similarly bounded by the fixed-size PhyGNN parameters, $\mathcal{O}(L\cdot d_h^2)$, together with the bounded local observation buffer whose size scales with $\Delta_{in}^{\max}$. For practice, the trained PhyGAIL actor requires only 2.78 MB of onboard storage for 723K parameters. Taken together, these three results show that PhyGAIL maintains constant communication, computation, and memory overhead per UAV during decentralized execution. This bound characterizes the asymptotic local decision cost under bounded observations, rather than the end-to-end wall-clock latency of the centralized simulator benchmark reported later in Section~V-D.

\section{Experimental Results}

This section first introduces the simulation setups and then discusses the performance of the proposed algorithm \footnote{The source codes of the proposed PhyGAIL algorithm are available on \textit{https://github.com/lytxzt/Physics-informed-GAIL}}. The experiment results of centralized algorithms like CR-MGC \cite{mou2022resilient}, DEMD \cite{lin2024multihop}, and GDR-TS \cite{wang2025topology}, decentralized approaches including center-fly, HERO \cite{mi2011hero}, SIDR \cite{chen2020sidr}, and CR-MA \cite{li2025fast}, are also displayed for comparisons.

\subsection{Experimental Setup}

\paragraph*{Simulation Environment}
All experiments are conducted in a 2-D UAV swarm simulator with communication radius $D_{\mathrm{comm}}=120$ m, maximum UAV speed $v_{\max}=10$ m/s, and control interval $\Delta t=0.1$ s. We consider five swarm sizes, $N\in\{20,50,100,200,500\}$, with square mission areas of side length $320$, $500$, $750$, $1000$, and $1600$ m, respectively. The maximum episode length is set to $T_{\max}=0.8W$, where $W$ denotes the map width. An episode terminates when the surviving UAVs reconnect into a single connected component or when the step budget is exhausted. During training, we impose a repulsion reward with a safety margin of $15$ m, while in evaluation a collision is counted when the pairwise distance falls below $10$ m.

\paragraph*{Baselines and Metrics}
For the main comparison, we vary the damage ratio $\rho=N_D/N$ from $0.05$ to $0.95$ with an interval of $0.05$, except for the $N=20$ case where the maximum ratio is $0.90$. For each swarm size, damage ratio, and method, we average the results over $50$ disconnected post-damage cases. The main evaluation metrics are the convergence rate, average recovery time, and average collisions per UAV; case-level variability is shown by confidence intervals in the curves and mean$\pm$std values in the aggregate tables. The runtime evaluation additionally reports the first-action latency, total solving time, and per-step inference time. All baselines strictly follow their original formulations and reported parameter settings. Centralized baselines such as CR-MGC, DEMD, and GDR-TS are evaluated with post-damage whole-topology access according to their original formulations. When these methods provide expert trajectories, the resulting supervision is restricted to offline expert selection and does not alter the local observation budget of PhyGAIL at test time.

\paragraph*{Training Details}
We train PhyGAIL on $20$-UAV scenarios and directly evaluate it on larger swarms without fine-tuning. Training uses $32$ parallel environments for $1000$ epochs under PPO with standard GAE settings, clipping coefficient $0.2$, and $5$ policy-update epochs per outer iteration. The entropy coefficient decays from $0.05$ to $0.005$. The actor, critic, and discriminator are optimized by AdamW with learning rates of $10^{-4}$, and the discriminator is updated once per outer epoch with label smoothing. The final training reward is computed as the environment reward plus $0.1$ times the GAIL reward, followed by normalization and clipping.

\begin{figure*}[!t]
\centering
\subfloat[Convergence rate under different swarm scales and damage ratios.\label{fig:convergence_large_panel}]{
    \includegraphics[width=0.98\textwidth]{fig/convergence_large_panel.png}
}\vspace{-0.5em}\\
\subfloat[Average recovery time under different swarm scales and damage ratios.\label{fig:recovery_large_panel}]{
    \includegraphics[width=0.98\textwidth]{fig/recovery_time_large_panel.png}
}\vspace{-0.5em}\\
\subfloat[Average collisions per UAV under different swarm scales and damage ratios.\label{fig:collision_large_panel}]{
    \includegraphics[width=0.98\textwidth]{fig/collision_large_panel.png}
}
\captionsetup{justification=raggedright,singlelinecheck=false}
\caption{Performance comparison under different swarm scales and damage ratios. Curves denote mean performance over $50$ disconnected post-damage cases for each setting. In panels (b) and (c), shaded bands indicate the corresponding $95\%$ confidence intervals.}
\label{fig:overall_comparison}
\end{figure*}

\paragraph*{Network Architecture}
The PhyGAIL encoder contains three physics-informed gated message-passing layers with hidden dimension $128$, and each node feature includes the relative position, velocity, and normalized topological degree. In graph construction, each active UAV retains up to $K_{act}=8$ active-neighbor communication edges and up to $K_{dmg}=3$ damaged-neighbor observation edges. The actor adopts a tanh-squashed Gaussian policy, the critic combines local graph features with masked global mean/max pooling, and the discriminator performs node-level expert classification conditioned on both state embeddings and actions.

\subsection{Overall Performance Comparison}

We first compare PhyGAIL with representative baselines under different swarm scales and damage ratios. Figure~\ref{fig:overall_comparison} summarizes the three primary metrics: convergence rate, average recovery time, and average collisions per UAV.

PhyGAIL maintains a perfect convergence rate across all tested swarm scales and damage ratios, as shown in Fig.~\ref{fig:convergence_large_panel}. This advantage becomes most evident in the large-scale settings ($N=100,200,500$), where HERO, SIDR, and CR-MA degrade rapidly as the damage ratio increases. Several centralized or centroid-seeking baselines, including center-fly, CR-MGC, DEMD, and GDR-TS, still preserve high convergence under low to moderate damage. Their behavior in the other two panels, however, shows that high convergence alone does not guarantee efficient or safe recovery. Among all compared methods, only PhyGAIL maintains perfect reconnection reliability over the full tested range.

Recovery time increases with the damage ratio for all methods, reflecting the increasing difficulty of reconnecting more severely fragmented swarms. In the large-scale settings, PhyGAIL stays in the best-performing group over most damage ratios and consistently outperforms the decentralized baselines in Fig.~\ref{fig:recovery_large_panel}. The shaded confidence bands stay relatively compact through low-to-moderate damage ratios and widen mainly in the most fragmented regimes, indicating that the dominant variability comes from scenario difficulty rather than unstable recovery behavior. Under extremely high damage ratios, some centralized baselines, especially DEMD, become slightly faster when reconnection is still feasible. This suggests that global topology information can help identify short merging paths in highly fragmented cases. Even so, this advantage appears only in a narrow high-damage regime, whereas PhyGAIL remains more robust over the full large-scale range and achieves the best overall average recovery time.

The collision curves in Fig.~\ref{fig:collision_large_panel} show a similarly clear trend. PhyGAIL keeps collisions near zero at small scales and maintains a much lower collision level than center-fly, CR-MGC, GDR-TS, and CR-MA in the large-scale cases. SIDR yields slightly lower collision values in part of the large-scale regime, but this comes with a substantial drop in convergence rate and recovery efficiency, which is consistent with a more conservative subnet-level motion strategy that avoids close encounters but often fails to drive timely reconnection under severe fragmentation. The collision-confidence bands remain narrow for PhyGAIL over most settings, which is consistent with stable collision-aware behavior. The main strength of PhyGAIL therefore lies in its overall balance: it combines reliable reconnection, competitive restoration speed, and low collision risk within a single decentralized policy.

\begin{table}[!t]
\centering
\small
\caption{Performance summary over large-scale settings ($N=100,200,500$).}
\label{tab:overall_summary}
\renewcommand{\arraystretch}{1.15}
\resizebox{\linewidth}{!}{
\begin{tabular}{ccccc}
\toprule
Algorithm & Conv. Rate$\uparrow$ & Rec. Time$\downarrow$ & Collision$\downarrow$ & Overall Rank$\downarrow$ \\
\midrule
PhyGAIL & \textbf{1.000} & \textbf{20.14$\pm$22.42} & 0.161$\pm$0.115 & \textbf{2.061} \\
DEMD & 0.990 & 26.55$\pm$22.84 & 0.964$\pm$0.419 & 3.956 \\
CR-MGC & 0.949 & 33.84$\pm$29.92 & 0.772$\pm$0.313 & 4.462 \\
GDR-TS & 0.975 & 36.48$\pm$32.84 & 0.931$\pm$0.537 & 4.471 \\
SIDR & 0.472 & 49.36$\pm$38.88 & \textbf{0.131$\pm$0.116} & 4.731 \\
center-fly & 0.998 & 28.94$\pm$25.12 & 1.173$\pm$0.622 & 4.798 \\
CR-MA & 0.606 & 49.39$\pm$40.81 & 0.333$\pm$0.176 & 4.904 \\
HERO & 0.276 & 67.56$\pm$43.51 & 0.809$\pm$0.584 & 6.617 \\
\bottomrule
\end{tabular}}
\end{table}

Table~\ref{tab:overall_summary} aggregates the results over the large-scale settings $N=\{100,200,500\}$. The overall rank is computed case by case and then averaged over all tested cases. PhyGAIL attains perfect average convergence ($1.000$), the shortest average recovery time ($20.14$ s), and the best overall rank ($2.061$). SIDR achieves a slightly smaller collision value, but this comes with a substantial drop in convergence rate and recovery efficiency.
This result should be read as an overall comparison rather than a claim of uniform dominance in every case: in a narrow subset of extremely fragmented but still recoverable cases, methods with full global topology access can occasionally find shorter merging paths, whereas PhyGAIL is stronger in overall balance and decentralized deployability across the tested range.


\begin{figure}[!t]
\centering
\includegraphics[width=\linewidth]{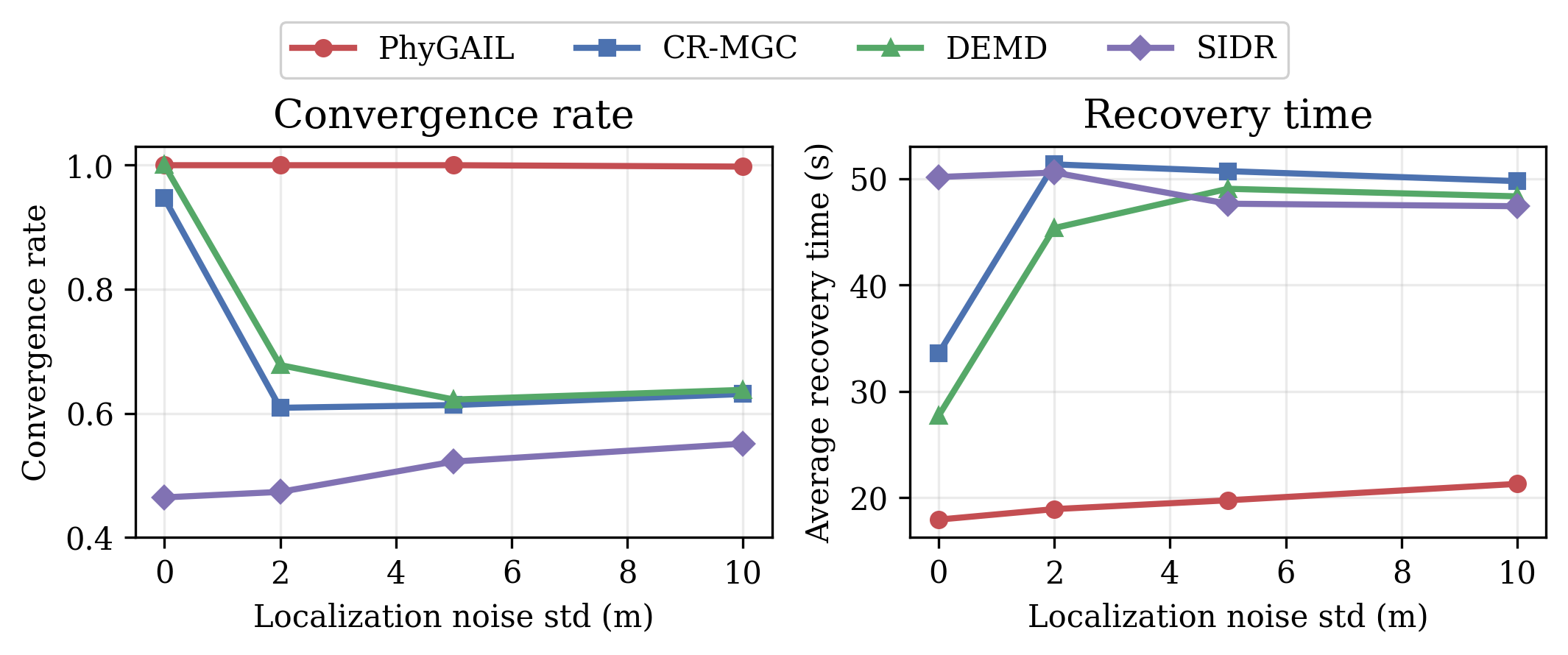}
\captionsetup{justification=raggedright,singlelinecheck=false}
\caption{Robustness under localization noise.}
\label{fig:realism_locnoise}
\end{figure}

\subsection{Robustness under Imperfect Observations}

We further examine whether PhyGAIL remains effective under imperfect local observations. Figure~\ref{fig:realism_locnoise} shows the large-scale localization-noise results aggregated over $N=\{100,200,500\}$ and $\rho=\{0.25,0.5,0.75\}$. The perturbation is applied to all observed position cues, including surviving UAVs, damaged UAVs, and the virtual-center reference of each local subnet, so the test covers noise on both neighborhood geometry and shared directional guidance. As the noise standard deviation increases from $0$ to $10$ m, PhyGAIL degrades gracefully: its aggregated convergence rate remains essentially unchanged ($1.000$ to $0.998$), while the average recovery time increases only moderately from $17.94$ s to $21.30$ s. By contrast, CR-MGC and DEMD suffer much larger drops in convergence rate and much larger increases in recovery time under the same perturbation, which suggests that the bounded local graph with relative kinematic observations is substantially less sensitive to spatial state noise than recovery policies that rely more heavily on accurate global topology snapshots. As a supplementary check, we also tested delayed local neighbor information for PhyGAIL. Across $0$ to $5$ delay steps, the convergence rate remained $1.000$ and the average recovery time changed only from $17.94$ s to $18.37$ s, so its effect was negligible and we omit a separate figure for brevity.

\subsection{Ablation Study}
We next identify which components contribute most to large-scale robustness and training stability.

\subsubsection{Scalability and Robustness}

\begin{table*}[!t]
\centering
\small
\caption{Ablation summary on large-scale settings.}
\label{tab:ablation_results}
\renewcommand{\arraystretch}{1.15}
\resizebox{0.98\textwidth}{!}{
\begin{tabular}{cccccccc}
\toprule
\multirow{2}{*}{Ablated Variant} & \multicolumn{2}{c}{$N=100$} & \multicolumn{2}{c}{$N=200$} & \multicolumn{2}{c}{$N=500$} & \multirow{2}{*}{Overall Rank$\downarrow$} \\
\cmidrule(lr){2-3}\cmidrule(lr){4-5}\cmidrule(lr){6-7}
& $\Delta$Conv. Rate$\uparrow$ & $\Delta$Rec. Time$\downarrow$ & $\Delta$Conv. Rate$\uparrow$ & $\Delta$Rec. Time$\downarrow$ & $\Delta$Conv. Rate$\uparrow$ & $\Delta$Rec. Time$\downarrow$ & \\
\midrule
w/o attraction gate & -0.267& +16.88 & -0.260& +22.27 & -0.700& +63.91 & 6.077 \\
w/o repulsion gate & -0.667& +37.48 & -0.833& +53.10 & -0.807& +75.01 & 7.185 \\
\midrule
w/ GAT instead of PhyGNN & -0.000& +5.70 & -0.007& +22.22 & -0.293& +51.43 & 6.457 \\
w/ GraphSAGE instead of PhyGNN & -0.067& +9.35 & -0.107& +11.85 & -0.260& +44.56 & 6.250 \\
w/ GCN instead of PhyGNN & -0.353& +19.36 & -0.447& +33.03 & -0.673& +64.90 & 8.216 \\
\midrule
w/o KNN selection & -0.007& +1.52 & -0.007& +7.01 & -0.113& +27.24 & 4.940 \\
w/o damaged neighbor & -0.040& +7.79 & -0.120& +19.96 & -0.433& +50.88 & 6.351 \\
w/o virtual center & -0.127& +13.80 & -0.347& +29.66 & -0.640& +65.69 & 7.294 \\
\midrule
w/o GAIL reward & -0.000& +1.27 & -0.000& +5.68 & -0.007& +20.18 & 4.591 \\
w/o centralized critic & -0.000& +2.43 & -0.000& +6.34 & -0.160& +38.82 & 5.198 \\
w/o expert time reward & -0.187& +15.40 & -0.513& +35.47 & -0.607& +57.84 & 7.429 \\
w/o any time reward & -0.553& +26.95 & -0.593& +35.87 & -0.667& +58.70 & 8.011 \\
\bottomrule
\end{tabular}}
\end{table*}

Table~\ref{tab:ablation_results} reports the performance drop of each variant relative to the full PhyGAIL model on the large-scale settings after averaging over the tested damage ratios $\rho=\{0.25,0.5,0.75\}$; the overall rank is again obtained case by case and then averaged over all tested cases.

The two gate ablations first confirm the need for complementary attraction and repulsion modeling. Removing either gate causes clear degradation across all three scales, while removing the repulsion gate is particularly harmful: both the convergence rate and the recovery time deteriorate sharply, indicating that attraction alone cannot maintain safe and effective large-scale coordination. Removing the attraction gate is less destructive but still causes large efficiency loss, especially at $N=500$, showing that repulsion alone is insufficient for reliable subnet merging.

The graph-encoder replacements then show the strongest effect on scalability. Replacing PhyGNN with conventional graph encoders consistently reduces convergence and recovery efficiency, and the gap widens with swarm size. The largest degradation occurs with the GCN variant, while the GAT and GraphSAGE variants remain more competitive at smaller scales but also degrade sharply at $N=500$. This pattern shows that standard graph aggregation is insufficient for stable large-scale recovery, whereas the physics-gated interaction in PhyGNN becomes more important as the topology grows more fragmented and heterogeneous.

Among the perception-related components, the virtual center has the largest impact on large-scale coordination, with its removal causing strong drops in both convergence and recovery efficiency across all three scales. Removing damaged-neighbor observations also degrades performance substantially, especially at $N=500$, while removing the KNN selection mechanism has only a mild effect until the largest scale. The bounded local graph is therefore useful for scalability, but the shared directional cue from the virtual center plays the larger role in driving subnet merging.

The training-related ablations finally highlight the importance of temporal reward design. Removing the expert time reward causes large degradation across all three scales, and removing all time-related rewards yields the worst overall rank. By comparison, removing the centralized critic or the GAIL reward changes the final large-scale averages much less. This pattern suggests that temporal reward design mainly determines large-scale restoration efficiency, whereas the centralized critic and GAIL reward contribute more to optimization quality than to the final asymptotic metrics.

\subsubsection{Training Stability}

\begin{figure*}[!t]
\centering
\subfloat[Training curves with GNN replacements.\label{fig:ablation_phygnn_train}]{
    \includegraphics[width=0.48\linewidth]{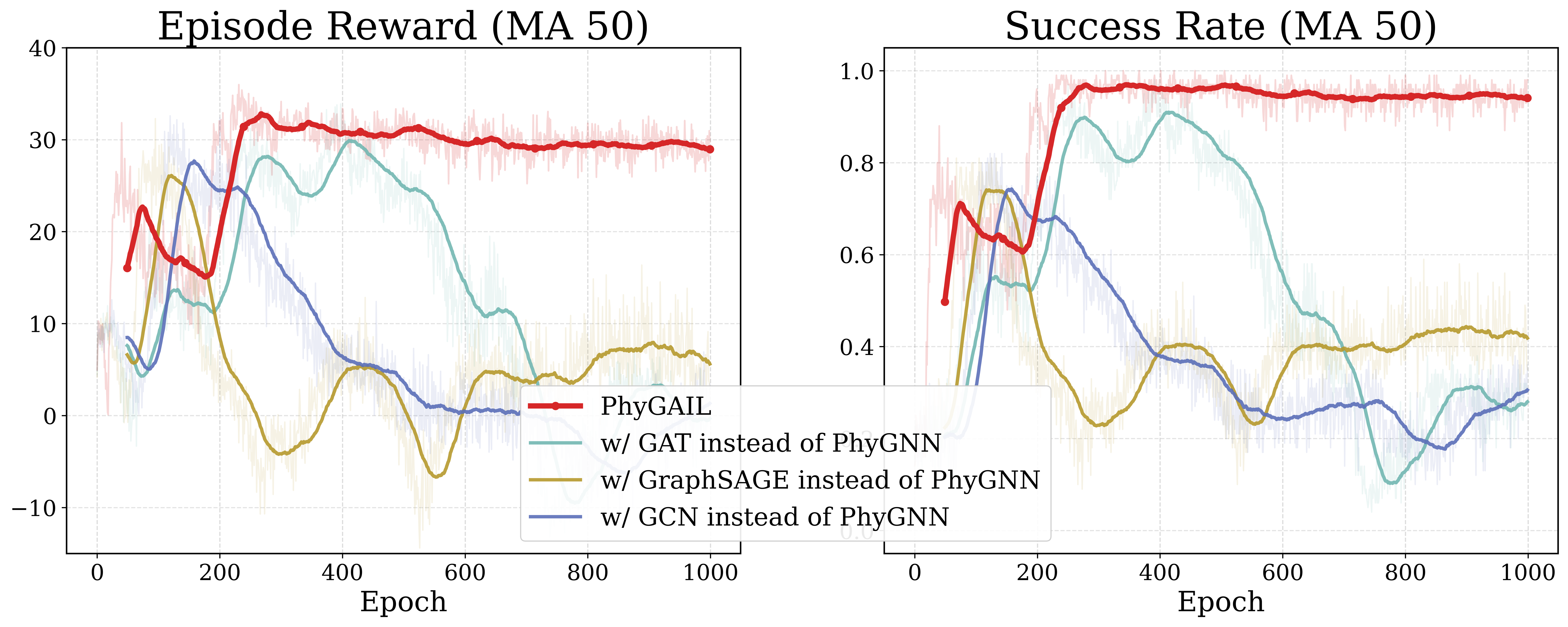}
}
\subfloat[Training curves with representative component ablations. \label{fig:ablation_component_train}]{
    \includegraphics[width=0.48\linewidth]{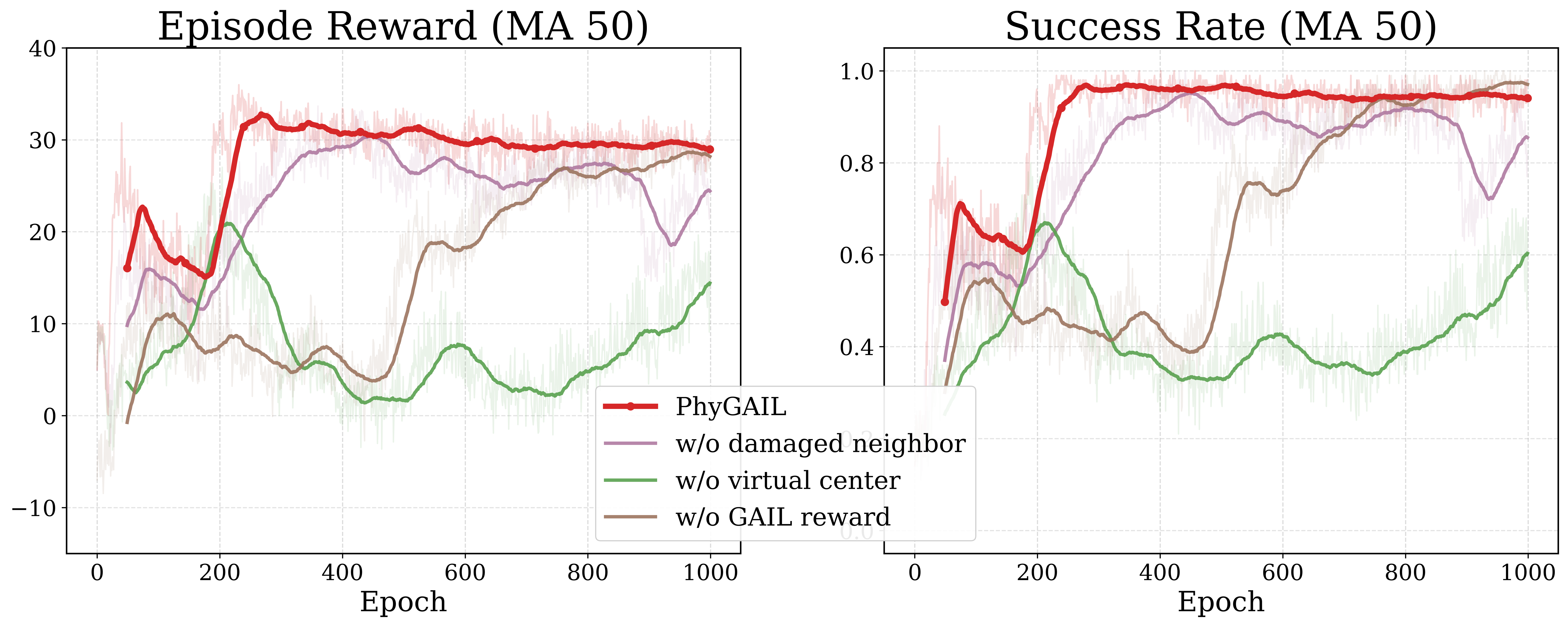}
}
\captionsetup{justification=raggedright,singlelinecheck=false}
\caption{Training stability of PhyGAIL and its ablated variants.}
\label{fig:ablation_train}
\end{figure*}

Figure~\ref{fig:ablation_train} shows the same separation in optimization dynamics. In Fig.~\ref{fig:ablation_phygnn_train}, the full PhyGAIL model quickly reaches a high-reward regime and maintains a high convergence level, whereas the GCN variant starts to collapse after its early improvement, the GAT variant reaches strong intermediate performance and then degrades in the later stage, and the GraphSAGE variant oscillates more strongly and recovers only partially after its drop. Figure~\ref{fig:ablation_component_train} further shows that removing damaged-neighbor observations weakens later-stage convergence, removing the virtual center causes much larger instability throughout training, and removing the GAIL reward mainly slows the early and middle stages without preventing eventual convergence. Together with Table~\ref{tab:ablation_results}, these curves indicate that PhyGNN and the virtual center govern large-scale robustness, the temporal reward design governs restoration efficiency, and the GAIL reward mainly accelerates and stabilizes training.

\subsubsection{Parameter Sensitivity}

\begin{figure}[!t]
\centering
\includegraphics[width=\linewidth]{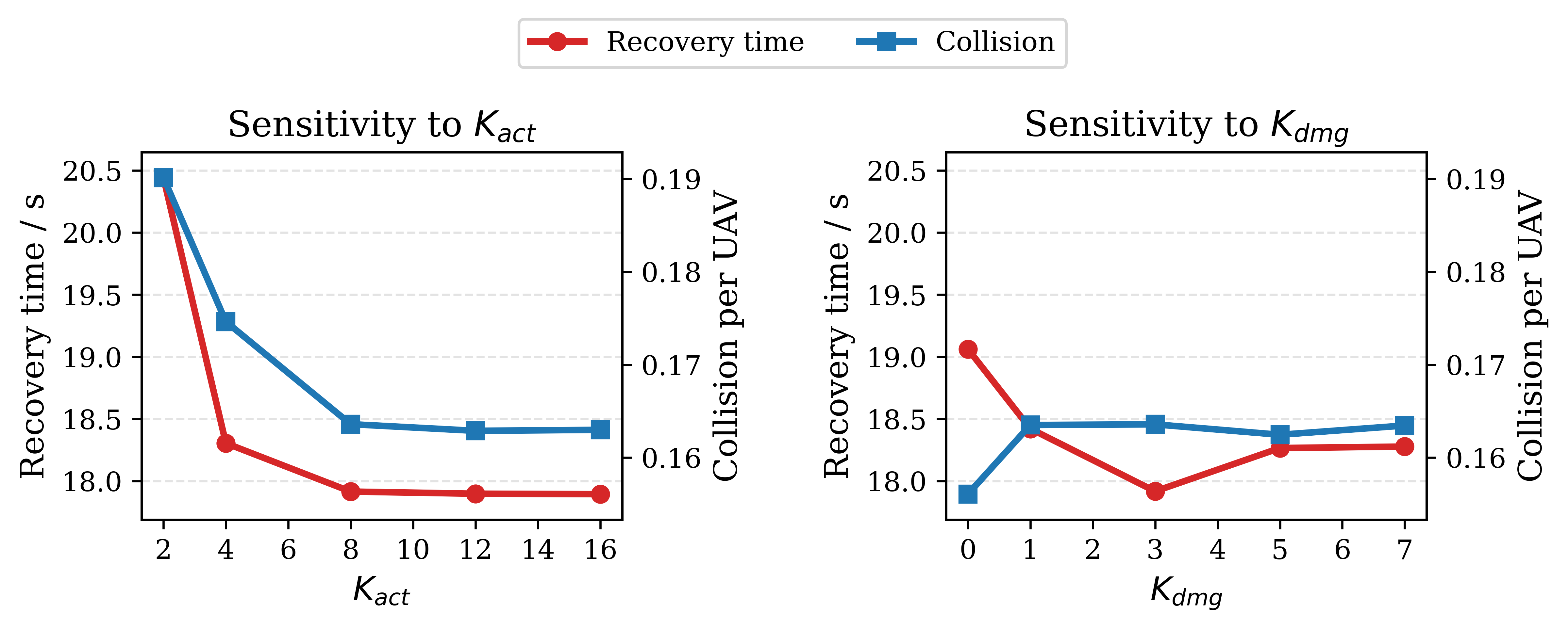}
\captionsetup{justification=raggedright,singlelinecheck=false}
\caption{Sensitivity to $K_{act}$ and $K_{dmg}$.}
\label{fig:k_sensitivity}
\end{figure}

Figure~\ref{fig:k_sensitivity} further examines the two bounded-perception parameters that determine the active-neighbor budget and damaged-node cue budget. As $K_{act}$ increases, the recovery time drops sharply from very small neighborhoods and then becomes almost unchanged beyond $K_{act}=8$, while the collision level also stabilizes in the same range. A similar pattern appears for $K_{dmg}$: very small damaged-node budgets are less efficient, but the performance difference becomes marginal once $K_{dmg}$ reaches $3$. Since larger $K_{act}$ and $K_{dmg}$ directly increase the local graph size and associated processing cost, we use the smallest values in this near-saturated region, namely $K_{act}=8$ and $K_{dmg}=3$.

\subsection{Case Study}

\begin{figure*}[!t]
\centering
\subfloat[Training dynamics of PhyGAIL.\label{fig:train1}]{
    \includegraphics[width=0.98\textwidth]{fig/case_study_train.png}
}\vspace{-0.5em}\\
\subfloat[Validation trajectories at representative epochs.\label{fig:train2}]{
    \includegraphics[width=0.98\textwidth]{fig/case_study_trajectories.png}
}
\captionsetup{justification=raggedright,singlelinecheck=false}
\caption{Case study of the learning and behavioral evolution of PhyGAIL.}
\label{fig:train}
\end{figure*}

We further inspect how the learned policy evolves during training and how this evolution appears in representative trajectories. Figure~\ref{fig:train1} presents the training dynamics of the full model, and Fig.~\ref{fig:train2} shows the corresponding validation trajectories at representative epochs.

The learning curves in Fig.~\ref{fig:train1} show a clear two-stage pattern. In the early stage, both the episode reward and the success rate rise quickly, which indicates that the policy first acquires basic collision avoidance and local coordination. A sharper transition then appears in the middle stage, after which both training and validation success rates approach $1.0$ and the episode reward climbs to a high plateau. The actor, critic, and discriminator losses follow the same stabilization trend, which suggests that the policy progresses from local motion regularities to an effective subnet-merging strategy.

The trajectories in Fig.~\ref{fig:train2} illustrate the same evolution on one validation case with $10$ damaged nodes and three disconnected surviving subnets. At epoch $20$, the subnet motions are no longer random, but they still lack a consistent merging direction and cannot reconnect within the horizon. At epoch $35$, the same case is solved in $130$ steps, showing that the policy has begun to use the shared directional cues encoded in the graph. This stage already captures the correct reconnection objective, although the subnet motions are still relatively indirect and only partially coordinated. By epoch $190$, the three subnets move toward the central region in a more synchronized way and solve the case in $68$ steps, while epoch $505$ further reduces the recovery time to $55$ steps mainly by removing redundant motion. Overall, the figure shows a progression from local motion learning, to expert cloning, then to stable reconnection behavior, and finally to more efficient large-scale coordination.

\begin{table*}[!t]
\centering
\small
\caption{Runtime comparison under large-scale settings ($N=100,200,500$).}
\label{tab:runtime_results}
\renewcommand{\arraystretch}{1.15}
\resizebox{0.98\textwidth}{!}{
\begin{tabular}{cccccccccc}
\toprule
\multirow{2}{*}{Algorithm} & \multicolumn{3}{c}{$N=100$} & \multicolumn{3}{c}{$N=200$} & \multicolumn{3}{c}{$N=500$} \\
\cmidrule(lr){2-4}\cmidrule(lr){5-7}\cmidrule(lr){8-10}
& Resp. (ms) & Solve (s) & Rec. Steps & Resp. (ms) & Solve (s) & Rec. Steps & Resp. (ms) & Solve (s) & Rec. Steps \\
\midrule
PhyGAIL & 4.02$\pm$0.23 & 0.14$\pm$0.08 & \textbf{52.7$\pm$32.3} & 7.42$\pm$0.58 & \textbf{0.82$\pm$0.38} & \textbf{143.5$\pm$66.2} & 17.61$\pm$1.24 & \textbf{6.60$\pm$3.02} & \textbf{269.5$\pm$122.7} \\
center-fly & \textbf{0.46$\pm$0.04} & \textbf{0.08$\pm$0.05} & 114.2$\pm$74.0 & \textbf{1.73$\pm$0.15} & 0.93$\pm$0.33 & 301.2$\pm$109.6 & \textbf{9.29$\pm$0.95} & 8.98$\pm$3.57 & 501.0$\pm$203.9 \\
HERO & 1.43$\pm$0.09 & 0.82$\pm$0.33 & 515.8$\pm$205.9 & 4.73$\pm$0.58 & 4.59$\pm$1.05 & 761.3$\pm$173.1 & 23.21$\pm$1.06 & 39.30$\pm$0.12 & 1280.0$\pm$0.0 \\
SIDR & 1.19$\pm$0.14 & 0.35$\pm$0.37 & 266.9$\pm$279.7 & 4.17$\pm$0.28 & 4.01$\pm$1.63 & 686.4$\pm$277.6 & 22.87$\pm$1.22 & 21.39$\pm$19.75 & 642.8$\pm$592.2 \\
CR-MGC & 511.93$\pm$384.61 & 0.58$\pm$0.38 & 94.4$\pm$53.7 & 556.10$\pm$130.87 & 1.28$\pm$0.21 & 238.6$\pm$63.8 & 4901.74$\pm$1532.01 & 17.18$\pm$3.87 & 680.9$\pm$182.5 \\
DEMD & 2862.48$\pm$620.27 & 2.87$\pm$0.65 & 100.1$\pm$61.1 & 2792.71$\pm$1262.92 & 3.40$\pm$1.05 & 223.7$\pm$70.4 & 16032.39$\pm$192.72 & 24.88$\pm$1.94 & 493.9$\pm$109.2 \\
GDR-TS & 9185.78$\pm$71.61 & 9.28$\pm$0.10 & 119.1$\pm$58.2 & 14404.36$\pm$427.61 & 15.29$\pm$0.46 & 287.3$\pm$31.8 & 69241.05$\pm$4248.93 & 80.73$\pm$4.50 & 641.6$\pm$51.7 \\
CR-MA & 14.96$\pm$0.26 & 3.27$\pm$2.27 & 201.2$\pm$139.5 & 46.75$\pm$1.02 & 24.86$\pm$11.13 & 487.9$\pm$218.4 & 247.65$\pm$1.57 & 241.10$\pm$62.62 & 924.1$\pm$239.5 \\
\bottomrule
\end{tabular}
}
\end{table*}

\subsection{Runtime Overhead Analysis}

We finally evaluate the computational overhead of PhyGAIL and compare it with representative baselines. Table~\ref{tab:runtime_results} lists the first-response latency, end-to-end solving time, and recovery steps of all methods under the large-scale settings $N=\{100,200,500\}$.

PhyGAIL keeps a low first-response latency of $4.02$, $7.42$, and $17.61$ ms for $N=\{100,200,500\}$, respectively, which remains far below the other learning-based baselines. The same advantage appears at the system level: PhyGAIL achieves the shortest solving time at $N=200$ and $N=500$, and stays close to the shortest at $N=100$ ($0.14$ s versus $0.08$ s for center-fly). Although center-fly is lighter per step, it requires many more recovery steps, especially at large scale, so its end-to-end solving time becomes larger. Overall, PhyGAIL provides the best runtime--performance trade-off among the compared methods for large-scale decentralized recovery.

The growth of measured response time with swarm size does not conflict with the theoretical $\mathcal{O}(1)$ complexity in Section~IV-E. The complexity analysis characterizes the local decision cost of one UAV under bounded local observations and bounded graph size, whereas the benchmarked latency is measured from the current centralized simulator implementation, which also includes subnet extraction, neighbor search, graph construction, batching, and action write-back for all active subnets. The observed latency growth therefore mainly reflects implementation overhead at the simulator level rather than a violation of the constant per-agent complexity result, and from the TMC perspective this bounded local processing is valuable because fragmented networks require fast recovery decisions before a whole-swarm topology view can be reassembled.

\section{Conclusion}

This paper considered decentralized connectivity restoration for UAV swarm networks under communication network split caused by large-scale node failures. We developed PhyGAIL, a CTDE-based framework that combines bounded heterogeneous graph perception, physics-informed graph interaction, and scenario-adaptive imitation learning for large-scale recovery. Its core PhyGNN module models local attraction and repulsion through gated message passing, while the accompanying structural analysis establishes bounded graph propagation, bounded pseudo-force generation, and controlled variance of the expert-normalized terminal reward. Simulation results showed that a policy trained on 20-UAV swarms transfers directly to swarms of up to 500 UAVs without fine-tuning, and achieves the best overall balance among reconnection reliability, recovery efficiency, motion safety, and runtime overhead. These results support PhyGAIL as a scalable solution for fragmented network settings in which post-damage information collection and per-agent online processing must both remain limited.

The present study still assumes a common 2-D flight layer, disk-like communication abstraction, a single damage event, and a shared pre-failure reference for decentralized recovery. Future work can consider more realistic deployment conditions, including continuous failures, obstacle-constrained environments, and tighter onboard resource limits, as well as broader continuous-space multi-agent coordination problems beyond connectivity restoration.

\appendices

\section{Proof of Proposition~\ref{prop:bounded_spectral_norm}}
\label{app:proof_spectral_norm}

\begin{IEEEproof}
In the directed perception graph $\mathcal{G}_{local}$, an edge $e_{j \to i}$ exists if node $u_i$ selects $u_j$ for information aggregation.

\emph{i. Algorithmic in-degree bound:}
The in-degree of node $u_i$ equals the number of neighbors from which it pulls information. By construction, $u_i$ can aggregate messages from at most $K_{act}$ active UAVs, $K_{dmg}$ damaged UAVs, and one virtual center. Therefore,
\begin{equation}
    \|\tilde{\bm{A}}\|_\infty \le K_{act} + K_{dmg} + 1.
\end{equation}

\emph{ii. Physical out-degree bound:}
The out-degree of node $u_j$ equals the number of agents that may select it as a neighbor. Inside a communication disk of radius $D_{comm}$, the number of UAVs that can coexist while respecting the minimum safety distance $D_{safe}$ is bounded by geometric packing. Hence,
\begin{equation}
    \|\tilde{\bm{A}}\|_1 \le C_{pack} = \mathcal{O}\left((D_{comm}/D_{safe})^d\right).
\end{equation}

For any asymmetric real matrix,
\begin{equation}
    \|\tilde{\bm{A}}\|_2 \le \sqrt{\|\tilde{\bm{A}}\|_1 \|\tilde{\bm{A}}\|_\infty}.
\end{equation}
Substituting the two bounds above gives
\begin{equation}
    \|\tilde{\bm{A}}\|_2 \le \sqrt{(K_{act} + K_{dmg} + 1)\cdot C_{pack}},
\end{equation}
which is constant with respect to the global swarm size $N$.
\end{IEEEproof}

\section{Proof of Proposition~\ref{prop:bounded_force}}
\label{app:proof_bounded_force}

\begin{IEEEproof}
The localized observation $\mathcal{O}_i(t)$ consists of normalized relative positions, velocities clipped by $v_{\max}$, and structural degrees mapped through a logarithmic function. These components lie in closed and bounded intervals, so the input space $\mathcal{X}$ is compact.

By the extreme value theorem, a continuous mapping on a compact set is bounded. The MLP layers in PhyGNN have finite parameters and continuous activations (ReLU, Sigmoid, and Softplus). Therefore, the interaction feature $\bm{f}_{ij}$ and the interaction strength $S_{ij}$ are bounded by constants $C_f$ and $S_{\max}$, respectively. Each directed message then satisfies
\begin{equation}
    \|\bm{m}_{j\rightarrow i}\| \le S_{\max} C_f.
\end{equation}

Proposition~\ref{prop:bounded_spectral_norm} bounds the number of incoming messages by $\Delta_{in}^{\max}$. Hence,
\begin{equation}
    \|\bm{m}_i\| \le \Delta_{in}^{\max} S_{\max} C_f.
\end{equation}
Therefore, the aggregated pseudo-force remains bounded.
\end{IEEEproof}

\section{Proof of Proposition~\ref{prop:variance_decoupling}}
\label{app:proof_variance_decoupling}

\begin{IEEEproof}
Because $\eta(t)$ is clipped to $[0,\eta_{\max}]$, the terminal success reward lies in the interval
\begin{equation}
    [r_{base},\; r_{base}+r_{speed}\eta_{\max}],
\end{equation}
whose length is $L=r_{speed}\eta_{\max}$.

Popoviciu's inequality states that any random variable bounded in an interval of length $L$ has variance at most $L^2/4$. Applying it here gives
\begin{equation}
    \mathrm{Var}(r_{term}^{+}) \le \frac{1}{4}(r_{speed}\eta_{\max})^2.
\end{equation}
\end{IEEEproof}

\bibliographystyle{IEEEtran}
\balance
\bibliography{IEEEabrv,main}


 




\vfill

\end{document}